\documentclass{article}
\usepackage{ijcai18}
\usepackage[framemethod=TikZ]{mdframed}

\usepackage{times, appendix}
\usepackage{xcolor}
\usepackage{soul}
\usepackage[utf8]{inputenc}
\usepackage[small]{caption}

\usepackage{graphicx}
\usepackage{float, hyperref}
\usepackage{amssymb,amstext,amsmath,amsthm}
\usepackage{enumerate}
\usepackage{xspace}
\usepackage{pifont} 
\usepackage{boxedminipage}
\usepackage{paralist}
\usepackage{thmtools,thm-restate}

\usepackage{tikz, tkz-graph, models, misc}
\usetikzlibrary{arrows, shapes}
\usepackage[lined,linesnumbered]{algorithm2e} 

\newcommand{\mc}[1]{\ensuremath{\mathcal{#1}}\xspace}

\renewcommand{\EL}{\mc{E\kern-0.1emL}}
\renewcommand{\ELI}{\mc{E\kern-0.1emLI}}

\newcommand{\Ind}{\mn{ind}}

\renewcommand{\phi}{\varphi}
\renewcommand{\epsilon}{\varepsilon}

\newcommand{\adom}{\ensuremath{\mathsf{dom}}}

\title{From Conjunctive Queries to Instance Queries in
  Ontology-Mediated Querying}

\author{
Cristina Feier$^1$, 
Carsten Lutz$^1$, 
Frank Wolter$^2$, 
\\ 
$^1$ Unversity of Bremen, Germany\\
$^2$ University of Liverpool, UK\\
feier@uni-bremen.de, 
clu@uni-bremen.de, 
wolter@liverpool.ac.uk
}

\begin{document}

\maketitle

\begin{abstract}
  We consider ontology-mediated queries (OMQs) based on expressive
  description logics of the \ALC family and (unions) of conjunctive
  queries, studying the rewritability into OMQs based on instance
  queries (IQs). Our results include exact characterizations of when
  such a rewriting is possible and tight complexity bounds for
  deciding rewritability. We also give a tight complexity bound for
  the related problem of deciding whether a given MMSNP sentence is
  equivalent to a CSP.
\end{abstract}

\section{Introduction}

An ontology-mediated query (OMQ) is a database-style query enriched
with an ontology that contains domain knowledge, aiming to deliver
more complete answers
\cite{DBLP:conf/rweb/CalvaneseGLLPRR09,DBLP:journals/tods/BienvenuCLW14,DBLP:conf/rweb/BienvenuO15}. In
OMQs, ontologies are often formulated in a description logic (DL) and
query languages of interest include conjunctive queries (CQs), unions
of conjunctive queries (UCQs), and instance queries (IQs).  While CQs
and UCQs are widely known query languages that play a fundamental role
also in database systems and theory, IQs are more closely linked
to~DLs.  In fact, an IQ takes the form $C(x)$ with $C$ a concept
formulated in the DL that is also used for the ontology, and thus the
expressive power of IQs depends on the ontology language. OMQs
based on (U)CQs are more powerful than OMQs based on IQs as the
latter only serve to return all objects from the data that are instances
of a given class.

It is easy to see that IQs can express tree-shaped CQs
with a single answer variable as well as unions thereof.  In fact,
this observation has been used in many technical constructions in the
area, see for example
\cite{DBLP:conf/pods/CalvaneseGL98,DBLP:journals/jair/GlimmLHS08,DBLP:conf/cade/Lutz08,DBLP:journals/jcss/EiterOS12}.
Intriguingly, though, it was observed by Zolin
\shortcite{DBLP:conf/dlog/Zolin07} that tree-shaped CQs are not the
limit of IQ-rewritability when we have an expressive DL such as \ALC
or \ALCI at our disposal. For example, the CQ $r(x,x)$, which asks to
return all objects from the data that are involved in a reflexive
$r$-loop, can be rewritten into the equivalent \ALC-IQ $P \rightarrow
\exists r. P(x)$. Here, $P$ behaves like a monadic second-order
variable due to the open-world assumption made for OMQs: we are free
to interpret $P$ in any possible way and when making $P$ true at an
object we are forced to make also $\exists r . P$ true if and only if
the object is involved in a reflexive $r$-loop. It is an interesting
question, raised in
\cite{DBLP:conf/dlog/Zolin07,DBLP:journals/japll/KikotZ13,DBLP:conf/dlog/KikotTZZ13},
to precisely characterize the class of CQs that are rewritable into
IQs. An important step into this direction has been made by Kikot and
Zolin~\shortcite{DBLP:journals/japll/KikotZ13} who identify a large
class of CQs that are rewritable into IQs: a CQ is rewritable into an
\ALCI-IQ if it is connected and every cycle passes through the (only)
answer variable; for rewritability into an \ALC-IQ, one additionally
requires that all variables are reachable from the answer variable in
a directed sense.  It remained open whether these classes are
depleting, that is, whether they capture all CQs that are IQ-rewritable.

There are two additional motivations to study the stated question. The
first one comes from concerns about the practical implementation of
OMQs. When the ontology is formulated in a more inexpressive `Horn
DL', OMQ evaluation is possible in \PTime data complexity and a host
of techniques for practically efficient OMQ evaluation is available,
see for example \cite{perezurbina10tractable,DBLP:conf/aaai/EiterOSTX12,DBLP:journals/ws/TrivelaSCS15,DBLP:conf/ijcai/LutzTW09}.
In the case of expressive DLs such as \ALC and \ALCI, OMQ evaluation
is {\sc coNP}-complete in data complexity and efficient implementation
is much more challenging. In particular, there are hardly any systems
that fully support such OMQs when the actual queries are (U)CQs. In
contrast, the evaluation of OMQs based on (expressive DLs and) IQs
is 
supported by several systems such as Pellet, Hermit, and PAGOdA
\cite{pellet,hermit,DBLP:journals/jair/ZhouGNKH15}. For this reason,
rewriting (U)CQs into IQs has been advocated in
\cite{DBLP:conf/dlog/Zolin07,DBLP:journals/japll/KikotZ13,DBLP:conf/dlog/KikotTZZ13}
as an approach towards efficient OMQ evaluation with expressive DLs
and (U)CQs. The experiments and optimizations reported in
\cite{DBLP:conf/dlog/KikotTZZ13} show the potential (and challenges)
of this approach. 

The second motivation stems from the connection between OMQs and
constraint satisfaction problems (CSPs)
\cite{DBLP:journals/tods/BienvenuCLW14,DBLP:journals/lmcs/LutzW17}. Let
$(\Lmc,\Qmc)$ denote the class of OMQs based on ontologies formulated
in the DL \Lmc and the query language \Qmc. It was observed in
\cite{DBLP:journals/tods/BienvenuCLW14} that $(\ALCI,\text{IQ})$ is
closely related to the complement of CSPs while $(\ALCI,\text{UCQ})$
is closely related to the complement of the logical generalization
MMSNP of CSP; we further remark that MMSNP is a notational variant of
the complement of (Boolean) monadic disjunctive Datalog. Thus,
characterizing OMQs from $(\ALCI,\text{UCQ})$ that are rewritable into
$(\ALCI,\text{IQ})$ is related to characterizing MMSNP sentences that
are equivalent to a CSP, and we also study the latter problem. In
fact, the main differences to the OMQ case are that unary queries are
replaced with Boolean ones and that predicates can have unrestricted arity.

The main aim of this paper is to study the rewritability of OMQs from
$(\Lmc,\text{(U)CQ})$ into OMQs from $(\Lmc,\text{IQ})$, considering
as \Lmc the basic expressive DL \ALC as well as extensions of \ALC
with inverse roles, role hierarchies, the universal role, and
functional roles. 
We provide precise characterizations, tight complexity bounds for
deciding whether a given OMQ is rewritable, and show how to construct
the rewritten query when it exists. In fact, we prove that the classes
of CQs from \cite{DBLP:journals/japll/KikotZ13} are depleting, but we
go significantly beyond that: while
\cite{DBLP:conf/dlog/Zolin07,DBLP:journals/japll/KikotZ13,DBLP:conf/dlog/KikotTZZ13}
aim to find IQ-rewritings that work for \emph{any} ontology, we
consider the more fine-grained question of rewriting into an IQ an OMQ
$(\Tmc,\Sigma,q(x))$ where \Tmc is a DL TBox formalizing the ontology,
$\Sigma$ is an ABox signature, and $q(x)$ is the actual query. The
`any ontology' setup then corresponds to the special case where \Tmc
is empty and $\Sigma$ is full. However, gaving a non-empty TBox or a
non-full ABox signature results in additional (U)CQs to become
rewritable. While we admit modification of the TBox during rewriting,
it turns out that this is mostly unnecessary: only in some rather
special cases, a moderate \emph{extension} of the TBox pays off.
All this requires non-trivial generalizations of the query classes and
IQ-constructions from \cite{DBLP:journals/japll/KikotZ13}. Our
completeness proofs involve techniques that stem from the connection
between OMQs and CSP such as a lemma about ABoxes of high girth due to
Feder and Vardi~\shortcite{DBLP:journals/siamcomp/FederV98}. The
rewritings we construct are of polynomial size when we work with the
empty TBox, but can otherwise become exponential in size.

Regarding IQ-rewritability as a decision problem, we show {\sc
  NP}-completeness for the case of the empty TBox. This can be viewed
as an underapproximation for the case with non-empty TBox and ABox
signature. With non-empty TBoxes, complexities are higher. When the
ABox signature is full, we obtain 2\ExpTime-completeness for DLs with
inverse roles and an \ExpTime lower bound and a \coNExpTime upper
bound for DLs without inverse roles. With unrestricted ABox signature,
the problem is 2\NExpTime-complete for DLs with inverse roles and
\NExpTime-hard (and in 2\NExpTime) for DLs without inverse roles. All
lower bounds hold for CQs and all upper bounds capture UCQs. We also
prove that it is 2\NExpTime-complete to decide whether a given MMSNP
sentence is equivalent to a CSP. This problem was known to be
decidable \cite{DBLP:journals/siamcomp/MadelaineS07}, but the
complexity was open.

We also consider $\mathcal{ALCIF}$, the extension of \ALCI with
functional roles, for which IQ-rewritability turns out to be undecidable
and much harder to characterize. We give a rather subtle
characterization for the case of the empty TBox and full ABox
signature and show that the decision problem is then decidable and {\sc
  NP}-complete. Since it is not clear how to apply CSP techniques, we
use an approach based on ultrafilters, starting from what was
done for \ALC without functional roles 
in \cite{DBLP:journals/japll/KikotZ13}.
 

Full proofs are in the appendix.

\section{Preliminaries}
\label{sec:prelim}

We use standard description logic notation and refer to
\cite{DBLP:books/daglib/0041477} for full details. In contrast to the
standard DL literature, we carefully distinguish between the
\emph{concept language} and the \emph{TBox language}. We consider four
concept languages. Recall that \emph{\ALC-concepts} are formed
according to the syntax rule
$$
  C,D ::= A \mid \neg C \mid C \sqcap D \mid C \sqcup D \mid \exists r
  . C \mid \forall r . C
$$
where $A$ ranges over \emph{concept names} and $r$ over \emph{role
  names}. As usual, we use $C \rightarrow D$ as an abbreviation for
$\neg C \sqcup D$. \emph{\ALCI-concepts} additionally admit the use of
\emph{inverse roles} $r^-$ in concept constructors $\exists r^-.C$ and
$\forall r^-.C$. With a \emph{role}, we mean a role name or an inverse
role. \emph{$\ALC^u$-concepts} additionally admit the use of the
\emph{universal role} $u$ in concept constructors $\exists u . C$ and
$\forall u . C$. In \emph{$\ALCI^u$-concepts}, both inverse roles and
the universal role are admitted.

We now introduce several TBox languages.  For $\Lmc$ one of the four
concept languages introduced above, an \emph{\Lmc-TBox} is a finite
set of \emph{concept inclusions} $C \sqsubseteq D$ where $C$ and $D$
are \Lmc concepts. So each concept language also serves as a TBox
language, but there are additional TBox languages of interest. We
include the letter \Hmc in the name of a TBox language to indicate
that \emph{role inclusions} $r \sqsubseteq s$ are also admitted in the
TBox and likewise for the letter \Fmc and \emph{functionality
  assertions} $\mn{func}(r)$ where in both cases $r,s$ are role names
or inverse roles in case that the concept language used admits inverse
roles. So it should be understood, for example, what we mean with an
$\ALCHI^u$-TBox and an \ALCFI-TBox.  As usual, the semantics is
defined in terms of interpretations, which take the form
$\Imc=(\Delta^\Imc,\cdot^\Imc)$ with $\Delta^\Imc$ a non-empty
\emph{domain} and $\cdot^\Imc$ an \emph{interpretation function}. An
interpretation is a \emph{model} of a TBox \Tmc if it satisfies all
inclusions and assertions in \Tmc, defined in the usual way. We write
$\Tmc \models r \sqsubseteq s$ if every model of \Tmc also satisfies
the role inclusion $r \sqsubseteq s$. 

An \emph{ABox} is a set of \emph{concept assertions} $A(a)$ and
\emph{role assertions} $r(a,b)$ where $A$ is a concept name, $r$ a
role name, and $a,b$ are \emph{individual names}. We use
$\mn{ind}(\Amc)$ to denote the set of all individual names that occur
in \Amc. An interpretation is a \emph{model} of an ABox \Amc if it
\emph{satisfies} all concept and role assertions in \Amc, that is, $a
\in A^\Imc$ when $A(a)$ is in \Amc and $(a,b) \in r^\Imc$ when
$r(a,b)$ is in \Amc. An ABox is \emph{consistent with a TBox} \Tmc if
\Amc and \Tmc have a common model.  A \emph{signature} $\Sigma$ is a
set of concept and role names. We use $\mn{sig}(\Tmc)$ to denote the
set of concept and role names that occur in the TBox \Tmc, and likewise for
other syntactic objects such as ABoxes.  A \emph{$\Sigma$-ABox} is an
ABox \Amc such that $\mn{sig}(\Amc) \subseteq \Sigma$. 

A \emph{conjunctive query (CQ)} is of the form $ q(\xbf) = \exists
\ybf\,\varphi(\xbf,\ybf), $ where \xbf and \ybf are tuples of
variables and $\varphi(\xbf,\ybf)$ is a conjunction of \emph{atoms} of
the form $A(x)$ or $r(x,y)$ with $A$ a concept name, $r$ a role name,
and $x,y \in \xbf \cup \ybf$.  We call \xbf the \emph{answer
  variables} of $q(\xbf)$ and \ybf \emph{quantified variables}.  For
purposes of uniformity, we use $r^-(x,y)$ as an alternative notation
to denote an atom $r(y,x)$ in a CQ.  In fact, when speaking about
\emph{an atom $r(x,y)$ in a CQ $q(\xbf)$}, we generally also include
the case that $r=s^-$ and $s(y,x)$ is the actual atom in $q(\xbf)$,
unless explicitly noted otherwise.  Every CQ $q(\xbf)=\exists \ybf \,
\vp(\xbf,\ybf)$ gives raise to a directed graph $G_q$ whose nodes are
the elements of $\xbf \cup \ybf$ and that contains an edge from $x$ to
$y$ if $\vp(\xbf,\ybf)$ contains an atom $r(x,y)$. The corresponding
undirected graph is denoted $G^u_q$ (it might contain self loops). We
can thus use standard terminology from graph theory to CQs, saying for
example that a CQ is \emph{connected}.
A \emph{homomorphism} from $q(\xbf)$ to an interpretation \Imc is a
function $h : \xbf \cup \ybf \to \Delta^\Imc$ such that $h(x) \in
A^\Imc$ for every atom $A(x)$ of $q(\xbf)$ and $(h(x),h(y)) \in
r^\Imc$ for every atom $r(x,y)$ of $q(\xbf)$.  We write $\Imc \models
q(\abf)$ and call \abf an \emph{answer to $q(\xbf)$ on \Imc} if there
is a homomorphism from $q(\xbf)$ to $\Imc$ with $h(\xbf) = \abf$.

A \emph{union of conjunctive queries (UCQ)} $q(\xbf)$ is a disjunction
of one or more CQs that all have the same answer variables~\xbf.  We
say that a UCQ is \emph{connected} if every CQ in it is. The
\emph{arity} of a (U)CQ is the number of answer variables in it.  For
$\Lmc \in \{ \ALC,\ALCI, \ALC^u, \ALCI^u \}$, an \emph{\Lmc-instance
  query (\Lmc-IQ)} takes the form $C(x)$ where $C$ is an \Lmc concept
and $x$ a variable. We write $\Imc \models C(a)$ if $a \in C^\Imc$.
All instance queries have arity~1.

%

An \emph{ontology-mediated query (OMQ)} takes the form
$Q=(\Tmc,\Sigma,q(\xbf))$ with \Tmc a TBox, $\Sigma \subseteq
\mn{sig}(\Tmc) \cup \mn{sig}(q)$ an ABox signature, and $q(\xbf)$ a
query.\footnote{The requirement $\Sigma \subseteq \mn{sig}(\Tmc) \cup
  \mn{sig}(q)$ is 
  harmless since symbols in the ABox that are not from $\mn{sig}(\Tmc) \cup
  \mn{sig}(q)$ do not affect answers.} The \emph{arity} of $Q$ is the arity of $q(\xbf)$ and $Q$ is 
\emph{Boolean} if it has arity zero. When
$\Sigma$ is $\mn{sig}(\Tmc) \cup \mn{sig}(q)$, then for brevity we
denote it with $\Sigma_{\mn{full}}$ and speak of the \emph{full ABox
  signature}.  Let \Amc be a $\Sigma$-ABox. A tuple $\abf \in
\mn{ind}(\Amc)$ is an \emph{answer} to $Q$ on \Amc if $\Imc \models
q(\abf)$ for all models \Imc of \Amc and~\Tmc.  We say that $Q$ is
\emph{empty} if for all $\Sigma$-ABoxes~\Amc, there is no answer to
$Q$ on \Amc.  Let $Q_1,Q_2$ be OMQs, $Q_i=(\Tmc_i,\Sigma,q_i(\xbf))$
for $i \in \{1,2\}$.  Then $Q_1$ is \emph{contained} in $Q_2$, written
$Q_1 \subseteq Q_2$, if for all $\Sigma$-ABoxes \Amc, every answer to
$Q_1$ on \Amc is also an answer to $Q_2$ on \Amc. Further, $Q_1$ and
$Q_2$ are \emph{equivalent}, written $Q_1 \equiv Q_2$, if $Q_1
\subseteq Q_2$ and $Q_2 \subseteq Q_1$.

We use $(\Lmc,\Qmc)$ to refer to the \emph{OMQ language} in which the
TBox is formulated in the language \Lmc and where the actual queries
are from the language \Qmc, such as in $(\ALCF,\text{UCQ})$. For
brevity, we generally write $(\Lmc,\text{IQ})$ instead of
$(\Lmc,\text{$\Lmc'$-IQ})$ when $\Lmc'$ is the concept language
underlying the TBox language $\Lmc$, so for example
$(\ALCHI,\text{IQ})$ is short for $(\ALCHI,\text{\ALCI-IQ})$.
%
  %
%
\begin{definition}
\label{def:general}
Let $(\Lmc,\Qmc)$ be an OMQ language. An OMQ $Q=(\Tmc,\Sigma,q(\xbf))$ is \emph{$(\Lmc,\Qmc)$-rewritable} if there is an OMQ $Q'$ from $(\Lmc,\Qmc)$ such that the answers to $Q$ and to $Q'$ are identical on any $\Sigma$-ABox that is consistent with \Tmc.
%
In this case, we say that $Q$ \emph{is rewritable} into $Q'$ and call $Q'$ a \emph{rewriting} of $Q$.
\end{definition}
%
Let $(\Lmc,\Qmc)$ be an OMQ-language. \emph{IQ-rewritability in
  $(\Lmc,\Qmc)$} is the problem to decide whether a given (unary) OMQ
$Q=(\Tmc,\Sigma,q(x))$ from $(\Lmc,\Qmc)$ is
$(\Lmc,\text{IQ})$-rewritable; for brevity, we simply speak of
\emph{IQ-rewritability} of $Q$ when this is the case.  The following
examples show that IQ-rewritability of $Q$ depends on several factors.
All claims made are sanctioned by results established in this paper.
%










\begin{example}
\label{ex:1}
(1)~IQ-rewritability depends on the topology of the actual query. Let
$q_1(x)=r(x,x)$. The OMQ
$(\emptyset,\Sigma_{\mn{full}},q_1(x))$ is rewritable into the OMQ
$(\emptyset,\Sigma_{\mn{full}},C(x))$ from $(\ALC,\text{IQ})$ where
$C$ is  $P \rightarrow \exists r . P.
$
In contrast, let $q_2(x)=\exists y \, s(x,y) \wedge r(y,y)$. The OMQ
$(\emptyset,\Sigma_{\mn{full}},q_2(x))$ is not rewritable into an OMQ from
$(\ALCI,\text{IQ})$.

(2)~If we are not allowed to extend the TBox, IQ-rewritability depends
on whether or not inverse roles are available.  Let $\Sigma = \{r,s\}$
and $q(x)=\exists y \, r(y,x) \wedge s(y,x)$.  The OMQ
$Q=(\emptyset,\Sigma,q(x))$ is rewritable into the OMQ
$(\emptyset,\Sigma,C(x))$ from $(\ALCI,\text{IQ})$ where
$C$ is
$P \rightarrow \exists r^- . \exists s . P.
$
$Q$ is also rewritable into the OMQ $(\Tmc,\Sigma, C'(x))$
from $(\ALC,\text{IQ})$ where
$
  \Tmc = \{ \exists s . P \sqsubseteq \forall r . P' \},
$
and
$C$ is $P \rightarrow P'$,
but it is not rewritable into any OMQ $(\Tmc,\Sigma,C''(x))$ from
$(\ALC,\text{IQ})$ with $\Tmc=\emptyset$. 

(3)~IQ-rewritability depends on the TBox. Let $q(x)=\exists x_1
\exists y_1 \exists y_2 \exists z \, A(x) \wedge r(x,x_1)\wedge
r(x_1,y_1) \wedge r(x_1,y_2) \wedge r(y_1,z) \wedge r(y_2,z) \wedge
B_1(y_1) \wedge B_2(y_2)$. 
The OMQ
$(\emptyset,\Sigma_{\mn{full}},q(x))$ is not rewritable into an OMQ
from $(\ALCI,\text{IQ})$.  Let $\Tmc = \{ A \sqsubseteq \exists
r . \exists r . (B_1 \sqcap B_2 \sqcap \exists r . \top) \}$. The OMQ
$(\Tmc,\Sigma_{\mn{full}},q(x))$ is rewritable into the OMQ
$(\Tmc,\Sigma_{\mn{full}},A(x))$ from $(\ALC,\text{IQ})$.

(4)~IQ-rewritability depends on the ABox signature. Let $q(x)$ be the
CQ from (3) without the atom $A(x)$ and let \Tmc be as in (3). The
OMQ $(\Tmc,\Sigma_{\mn{full}},q(x))$ is not rewritable into an OMQ
from $(\ALCI,\text{IQ})$. Let $\Sigma = \{ A \}$. The OMQ
$(\Tmc,\Sigma,q(x))$ is rewritable into the OMQ
$(\Tmc,\Sigma,A(x))$ from $(\ALC,\text{IQ})$. 
\end{example}
Note that we are allowed to completely rewrite the TBox when
constructing IQ-rewritings, which might seem questionable from a
practical perspective.  Fortunately, though, it turns out the TBox can
always be left untouched or, in some rare cases, only needs to be
slightly extended. Also note that an alternative 
definition of IQ-rewritability obtained by dropping the restriction to
ABoxes consistent with \Tmc in Definition~\ref{def:general}. All
results obtained in this paper hold under both definitions.  We
comment on this throughout the paper 
and refer to the alternative version as \emph{unrestricted
  IQ-rewritability}.

\section{Characterizations}
\label{sect:charact}


We aim to provide characterizations of OMQs that are IQ-rewritable.
On the one hand, these characterizations clarify which OMQs are
IQ-rewritable and which are not. On the other hand, they form the
basis for deciding IQ-rewritability. We first concentrate on the case
of DLs (and IQs) with inverse roles and then move on to DLs without
inverse roles. In the final part of this section, we consider the case
where the TBox is empty, both with and without inverse roles.

\subsection{The Case With Inverse Roles}

To state the characterization, we need some preliminaries.  Let $q(x)$
be a CQ.  
%
%
A \emph{cycle} in $q(x)$ is a sequence of non-identical atoms
$r_0(x_0,x_1),\dots,r_{n-1}(x_{n-1},x_n)$ in $q(x)$, \mbox{$n \geq 1$},
where\footnote{We require the atoms be non-identical to
  prevent $r(x_0,x_1)$, $r^-(x_1,x_0)$ from being a cycle (both atoms are
  identical).}
 \begin{enumerate}

 \item $r_0,\dots,r_{n-1}$ are (potentially inverse) roles,

 \item $x_i \neq x_j$ for $0 \leq i < j < n$, and $x_0=x_n$.

 \end{enumerate}
 The \emph{length} of this cycle is $n$.  We say that $q(x)$
 is 
 \emph{$x$-acyclic} if every cycle in it passes through $x$ and use
 $q^{\mn{con}}(x)$ to denote the result of restricting $q(x)$ to those
 atoms that only use variables reachable in $G^u_q$ from $x$.  Both
 notions are lifted to UCQs by applying them to every CQ in the UCQ.
 A \emph{contraction} of $q(x)$ is a CQ obtained from $q(x)$ by zero
 or more variable identifications, where the identification of $x$
 with any other variable yields~$x$.

 Let $\Tmc$ be an $\ALCHI^u$-TBox and $q(x)$ a UCQ. We use
 $q_{\mn{acyc}}(x)$ to denote the UCQ that consists of all $x$-acyclic
 CQs obtained by starting with a contraction of a CQ from $q(x)$ and
 then replacing zero or more atoms $r(y,z)$ with $s(y,z)$ when $\Tmc
 \models s \sqsubseteq r$. 
 We write $q^{\mn{con}}_{\mn{acyc}}(x)$
 to denote $(q_{\mn{acyc}})^{\mn{con}}(x)$.
\begin{restatable}{theorem}{thmcentral}
\label{thm:central}
Let $\Lmc \in \{ \ALCI, \ALCHI \}$ and let $Q=(\tbox, \Sigma,q(x))$ be
a unary OMQ from $(\Lmc, \text{UCQ})$ that is non-empty. Then the
following are equivalent:
  \begin{enumerate}

  \item $Q$ is IQ-rewritable, that is, it is rewritable into an OMQ $Q'=(\Tmc', \Sigma,
    C(x))$ from $(\Lmc,\text{IQ})$;

  \item $Q$ is rewritable into an OMQ $Q'=(\tbox, \Sigma, C(x))$ from~$(\Lmc,\text{IQ})$;


  \item $Q \equiv (\Tmc, \Sigma, q^{\mn{con}}_{\mn{acyc}}(x))$.

\end{enumerate} 
%
When \Lmc is replaced with $\Lmc^u$, then the same equivalences hold
except that $q^{\mn{con}}_{\mn{acyc}}$ is replaced with
$q_{\mn{acyc}}$. 
%
\end{restatable}
Note that Theorem~\ref{thm:central} excludes empty OMQs, but these are
trivially IQ-rewritable. It implies that, in the considered cases, it
is never necessary to modify the TBox when constructing an
IQ-rewriting. Further, it emerges from the proof that it is never
necessary to introduce fresh role names in the rewriting (while fresh
concept names are crucial).  Theorem~\ref{thm:central} also applies to
unrestricted IQ-rewritability (where also ABoxes are admitted that are
inconsistent with the TBox from the OMQ): unrestricted
IQ-rewritability trivially implies IQ-rewritability and the converse
is an easy consequence of the fact that every OMQ that is
IQ-rewritable has an IQ-rewriting based on the same TBox.

We now give some ideas about the proof of Theorem~\ref{thm:central}.
The most interesting implication is ``$1 \Rightarrow 3$''
.  A central step is to show that if $Q=(\Tmc,\Sigma,q(x))$ is
IQ-rewritable into an OMQ $Q'$, then $Q \subseteq Q_{\mn{acyc}} := 
(\Tmc,\Sigma,q_{\mn{acyc}}(x))$, that is, when $\Amc \models Q(a)$ for
some $\Sigma$-ABox \Amc, then $\Amc \models Q_{\mn{acyc}}(a)$. To this
end, we first construct from $\Amc$ a $\Sigma$-ABox $\Amc^g$ of high
girth (that is, without small cycles) in a way such that (a)~$\Amc^g$
homomorphically maps to $\Amc$ and (b)~from $\Amc \models Q'(a)$ it
follows that $\Amc^g \models Q'(a)$, thus $\Amc^g \models Q(a)$. Due
to the high girth of $\Amc^g$ and exploiting (a variation of the) tree
model property for $\ALCHI$, we can then show that $\Amc^g \models
Q(a)$ implies $\Amc^g \models Q_{\mn{acyc}}(a)$.  Because of (a), it
follows that $\Amc \models Q_{\mn{acyc}}(a)$. In the direction ``$3
\Rightarrow 2$'', we construct actual rewritings, based on the
following lemma, an extension of a result of Kikot and Zolin
\cite{DBLP:journals/japll/KikotZ13} with TBoxes and ABox signatures
(and UCQs instead of CQs).
%
%

\begin{restatable}{lemma}{lemrewrALCI}
\label{lem:rewrAlgALCI} 
Let $Q=(\tbox, \Sigma, q(x))$ be an OMQ from $(\ALCHI^u,
\text{UCQ})$. Then 
\begin{enumerate}
\item if $q(x)$ is $x$-acyclic and connected, then $Q$ is rewritable into an OMQ $(\tbox, \Sigma, C(x))$ with $C(x)$ an \ALCI-IQ and 

\item if $q(x)$ is $x$-acyclic, then $Q$ is rewritable into an OMQ  $(\tbox, \Sigma, C(x))$ with $C(x)$ an $\ALCI^u$-IQ. 
\end{enumerate}
The size of the IQs $C(x)$ is polynomial in the size of $q(x)$.
\end{restatable}
We give the construction of the \ALCI-IQ $q'(x)$ in Point~1 of
Lemma~\ref{lem:rewrAlgALCI}. Let $Q=(\tbox, \Sigma, q(x))$ be an OMQ
from $(\ALCHI^u, \text{UCQ})$ with $q(x)$ $x$-acyclic and
connected. To construct $q'(x)$, we first construct for each CQ $p(x)$
in $q(x)$ an \emph{\ELI-concept} $C_p$, that is, an \ALCI-concept that
uses only the constructors $\sqcap$, $\exists r . C$, and $\exists r^-
. C$. In fact, since $p(x)$ is $x$-acyclic and connected, we can
repeatedly choose and remove atoms of the form $r(x,y)$ that occur in
a cycle in $p(x)$ and will eventually end up with a tree-shaped CQ
$p'(x)$.\footnote{Note that $x$ is the answer variable and recall
  that we might have $r=s^-$ and thus also choose atoms $s(y,x)$.} Here,
\emph{tree-shaped} means that the undirected graph $G^u_{p'}$ is a
tree and that there are no multi-edges, that is, if $r(y,z)$ is an
atom, then there is no atom $s(y,z)$ with $s \neq r$.  Next, extend
$p'(x)$ to obtain another tree-shaped CQ $p''(x)$ by taking a fresh
concept name $P \notin
\Sigma$, 
and adding $r(x',y)$ and
$P(x')$ for each removed atom $r(x,y)$, $x'$ 
a fresh
variable. We can now view $p''(x)$ as an \ELI-concept $C_p$ in the
obvious way. The desired \ALCI-IQ $q'(x)$ is $(P \rightarrow
\bigsqcup_{p(x) \text{ a CQ in } q(x)} C_p)(x).$ 

\subsection{The Case Without Inverse Roles}

We consider OMQs whose TBoxes are formulated in a DL \Lmc that does
not admit inverse roles
. Note that
inverse roles are then also not admitted in the IQ used in the
rewriting. We first observe that this has less impact than one might
expect: inverse roles in the IQ-rewriting can be eliminated and in
fact Points~1 and~3 from Theorem~\ref{thm:central} are still
equivalent. However, there is also a crucial difference: unless the
universal role is present, the elimination of inverse roles requires
an extension of the TBox and thus the equivalence of Points~1 and~2 of
Theorem~\ref{thm:central} fails. In fact, this is illustrated by
Point~(2) of Example~\ref{ex:1}. We thus additionally characterize
IQ-rewritability without modifying the TBox. We also show that, with
the universal role, it is not necessary to extend the TBox.

We start with some preliminaries.  An \emph{extended conjunctive query
  (eCQ)} is a CQ that also admits atoms of the form $C(x)$, $C$ a
(potentially compound) concept, and \emph{UeCQs} and \emph{extended
  ABoxes (eABoxes)} are defined analogously. The semantics 
is defined in the expected way. Every eCQ $q(x)$ gives
rise to an eABox $\Amc_q$ 
by viewing the
variables in $q(x)$ as individual names and the atoms as assertions.


Let $q(x)$ be an eCQ. We use $\mn{dreach}(q)$ to denote the set of all
variables reachable from $x$ in the directed graph $G_q$ and say that
$q(x)$ is \emph{$x$-accessible} if $\mn{dreach}(q)$ contains all
variables. For $V$ a set of variables from $q(x)$ that includes $x$,
$q(x)|_V$ denotes the restriction of $q(x)$ to the atoms that use only
variables from $V$.


Let \Tmc be an \ALC-TBox. An eCQ $p(x)$ is a \emph{\Tmc-decoration}
of a CQ $q(x)$ if 
\begin{enumerate}

\item $p(x)$ is obtained from $q(x)$ by adding, for each $y \in
  \mn{dreach}(q)$ and each subconcept $C$ of \Tmc, the atom $C(y)$ or
  the atom $\neg C(y)$;

\item the eABox $\Amc_{p}$ is consistent with \Tmc.

\end{enumerate}
For a UCQ $q(x)$, we use $q^{\mn{deco}}(x)$ to denote the UeCQ that
consists of all eCQs $p(x)|_{\mn{dreach}(p)}(x)$, where $p(x)$ is a
\Tmc-decoration of a CQ from $q(x)$. We write
$q^{\mn{deco}}_{\mn{acyc}}(x)$ to denote
$(q_{\mn{acyc}})^{\mn{deco}}(x)$. We now give the results announced above.
%
%
 %
 %
\begin{restatable}{theorem}{thmalc} 
\label{thm:alc}
Let $\Lmc \in \{ \ALC,\ALCH \}$ and let $Q=(\tbox, \Sigma, q(x))$ be a unary OMQ from $(\Lmc, \text{UCQ})$ that is non-empty. Then the following are equivalent:
  \begin{enumerate}

  \item $Q$ is rewritable into an OMQ from 
     $(\Lmc, \text{IQ})$;
  \item $Q$ is rewritable into an OMQ $(\Tmc \cup \Tmc',\Sigma,C(x))$
    from $(\Lmc,\text{IQ})$;
 \item $Q$ is rewritable into an OMQ from $(\Lmc\Imc, \text{IQ})$;
\end{enumerate}
If $\Sigma=\Sigma_{\mn{full}}$, then the following are
equivalent: 
  \begin{enumerate}
  \setcounter{enumi}{3}
  \item $Q$ is rewritable into an OMQ $Q'=(\tbox, \Sigma_{\mn{full}}, C(x))$ from $(\Lmc,\text{IQ})$;
  \item $Q \equiv (\Tmc,\Sigma_{\mn{full}}, q^{\mn{deco}}_{\mn{acyc}}(x))$.
\end{enumerate}
If, furthermore, $\Lmc$ is replaced with $\Lmc^u$ and $\Lmc\Imc$ with
$\Lmc\Imc^u$, then Conditions~1 to~3 are further equivalent to:
\begin{enumerate}
  \setcounter{enumi}{5}
  \item $Q$ is rewritable into an OMQ $Q'=(\tbox, \Sigma, C(x))$ from $(\Lmc^u,\text{IQ})$.
\end{enumerate}
\end{restatable}
Characterizing IQ-rewritability in the case where $\Lmc \in \{
\ALC,\ALCH \}$, the TBox (is non-empty and) cannot be extended, and
$\Sigma \neq \Sigma_{\mn{full}}$ remains an open problem.

\medskip

In the directions ``$3 \Rightarrow 2$'', ``$5 \Rightarrow 4$'', and
``$3 \Rightarrow 6$'', we have to construct IQ-rewritings. This is
done by starting with the rewriting from the proof of
Lemma~\ref{lem:rewrAlgALCI} and then modifying it appropriately.  As
in the case of Theorem~\ref{thm:central}, it is straightforward to see
that all results stated in Theorem~\ref{thm:alc} also apply
to unrestricted IQ-rewritability.



\subsection{The Case of Empty TBoxes}

We consider OMQs in which the TBox is empty as an important special
case. Since it is then not interesting to have an ABox signature, this
corresponds to the rewritability of (U)CQs into \Lmc-instance queries,
for some concept language~\Lmc (and thus no OMQs are involved). The
importance of this case is due to the fact that it provides an
`underapproximation' of the IQ-rewritability of OMQs, while also being
easier to characterize and computationally simpler.

We say that an UCQ $q(x)$ is \emph{\Lmc-IQ-rewritable} if there is an \Lmc-IQ $q'(x)$ that is equivalent to $q(x)$ in the sense that the
OMQs $(\emptyset, \Sigma_{\mn{full}},q(x))$ and $(\emptyset,
\Sigma_{\mn{full}},q'(x))$ are equivalent (and in passing, we define the
equivalence between two UCQs in exactly the same way). The following
proposition makes precise what we mean by underapproximation.
\begin{restatable}{proposition}{lememptyasapprox}
\label{lem:emptyasapprox} 
Let $\Lmc \in \{ \ALC, \ALCI, \ALC^u, \ALCI^u \}$.  If a UCQ $q(x)$ is
\Lmc-IQ-rewritable, then so is any OMQ $(\Tmc,\Sigma,q(x))$ from
$(\Lmc\Hmc, \text{UCQ})$.
\end{restatable}
Proposition~\ref{lem:emptyasapprox} is essentially a corollary of
Theorem~\ref{thm:centralnoTBox} below. As illustrated by Case~(3) of
Example~\ref{ex:1}, its converse fails.

We now characterize IQ-rewritability in the case of the empty TBox. A
\emph{subquery} of a CQ $q(x)$ is a CQ $q'(x)$ obtained from $q(x)$ by
dropping atoms. A \emph{subquery} of a UCQ $q(x)$ is a UCQ obtained by
including as a CQ at most one subquery of each CQ in~$q(x)$.
\begin{restatable}{theorem}{thmemptyTBox}
\label{thm:centralnoTBox} 
Let $q(x)$ be a UCQ. Then
  \begin{enumerate}
\item $q(x)$ is rewritable into an \ALCI-IQ iff there is a subquery $q'(x)$ of $q(x)$ that is $x$-acyclic, connected, and equivalent to $q(x)$;
\item $q(x)$ is rewritable into an \ALC-IQ iff there is a subquery $q'(x)$ of $q(x)$ that is $x$-acyclic, $x$-accessible, and equivalent to $q(x)$.
\end{enumerate} 
When \Lmc-IQs are replaced with $\Lmc^u$-IQs, then the same
equivalences hold except that connectedness/$x$-accessibility is dropped.
\end{restatable}

%
%
%
%
%
Note that Theorem~\ref{thm:centralnoTBox} also characterizes
rewritability of CQs; the query $q'(x)$ is then also a CQ rather than
a UCQ. This is in contrast to Theorems~\ref{thm:central}
and~\ref{thm:alc} where the queries $q^{\mn{con}}_{\mn{acyc}}(x)$ and
$q^{\mn{deco}}_{\mn{acyc}}(x)$ are UCQs even when the query $q(x)$
from the OMQ that we start with is a CQ. Another crucial difference is
that $q^{\mn{con}}_{\mn{acyc}}(x)$ and $q^{\mn{deco}}_{\mn{acyc}}(x)$
can be of size exponential in the size of the original OMQ while the
query $q'(x)$ in Theorem~\ref{thm:centralnoTBox} is of size polynomial
in the size of~$q(x)$.

\section{Complexity}
\label{sect:compl}

We determine the complexity of deciding IQ-rewritability in various
OMQ languages, based on the established characterizations and starting
with the case of empty TBoxes.
%
%
%
\begin{restatable}{theorem}{thmemptyTComplexity}
\label{thm:emptyTComplexity}
  For every $\Qmc \in \{ CQ, UCQ \}$ and $\Lmc \in \{ \ALC, \ALCI,
  \ALC^u, \ALCI^u \}$, it is {\sc NP}-complete to decide whether a given
  query from \Qmc is \Lmc-IQ-rewritable.
\end{restatable}
The upper bound in Theorem~\ref{thm:emptyTComplexity} is by 
guessing the query $q'(x)$ from Theorem~\ref{thm:centralnoTBox} 
and verifying that it satisfies the properties stated there. The lower
bound is by a reduction from 3-colorability. 

We next consider the case where TBoxes can be non-empty, starting
with the assumption that the ABox signature is full since this results
in (slightly) lower complexity.
\begin{restatable}{theorem}{thmfullSigComplexity}
\label{thm:fullSigComplexity}
Let $\Qmc \in \{ \text{CQ}, \text{UCQ} \}$. For OMQs based on the
full ABox signature, IQ-rewritability is
  \begin{enumerate}

  \item \ExpTime-hard in $(\ALC,\Qmc)$ and in \coNExpTime in
    $(\ALCH,\Qmc)$ and

  \item 2\ExpTime-complete in $(\ALCI,\Qmc)$ and $(\ALCHI,\Qmc)$.

  \end{enumerate}
\end{restatable}
The lower bounds are by reduction from 
OMQ evaluation on ABoxes of the form $\{ A(a) \}$, $A$ a concept name,
which is \ExpTime-complete in $(\ALCH,\text{CQ})$ and
2\ExpTime-complete in $(\ALCHI,\text{CQ})$
\cite{DBLP:conf/cade/Lutz08}. The upper bounds are derived from the
OMQ containment checks suggested by Condition~3 of
Theorem~\ref{thm:central} and Condition~4 of Theorem~\ref{thm:alc}.
Since we work with the full ABox signature, the non-emptiness
condition from these theorems is void (there are no empty OMQs)
and OMQ containment is closely related to OMQ evaluation, which allows
us to derive upper bounds for the former from the latter; in fact,
these bounds are exactly the ones stated in
Theorem~\ref{thm:fullSigComplexity}. We have to exercise some care,
for two reasons: first, we admit UCQs as the actual query and thus the
trivial reduction of OMQ containment to OMQ evaluation that is
possible for CQs (which can be viewed as an
ABox) does not apply. And second, we aim for upper bounds that exactly
match the complexity of OMQ containment while the UCQs
$q^{\mn{con}}_{\mn{acyc}}(x)$ and $q^{\mn{deco}}_{\mn{acyc}}(x)$
involved in the containment checks are of exponential size. What
rescues us is that each of the CQs in these UCQs is only of polynomial
size.

We finally consider the case where the ABox signature is unrestricted.
\begin{restatable}{theorem}{thmPSigComplexity}
\label{thm:PSigComplexity}
  IQ-rewritability is
  \begin{enumerate}

  \item \NExpTime-hard in $(\ALC,\text{CQ})$ and

  \item 2\NExpTime-complete in all of $(\ALCI,\text{CQ})$,
    $(\ALCI,\text{UCQ})$,
    $(\ALCHI,\text{CQ})$, $(\ALCHI,\text{UCQ})$.

  \end{enumerate}
\end{restatable}
The lower bound in Point~1 is by reduction from OMQ emptiness in
$(\ALC,\text{CQ})$, which is \NExpTime-complete
\cite{DBLP:journals/jair/BaaderBL16}. For the one in Point~2, we use a
reduction from OMQ containment, which is 2\NExpTime-complete in
$(\ALCI,\text{CQ})$ \cite{DBLP:conf/kr/BourhisL16}. The upper bounds
are obtained by appropriate containment checks as suggested by our
characterizations, and we again have to deal with UCQs with
exponentially many CQs. Note that Theorem~\ref{thm:PSigComplexity}
leaves open the complexity of IQ-rewritability in $(\ALC,\text{CQ})$,
between \NExpTime and 2\NExpTime. The same gap exists for OMQ
containment \cite{DBLP:conf/kr/BourhisL16} as well as in the related
problems of FO-rewritability and
Datalog-rewritability~\cite{DBLP:conf/icdt/FeierKL17}.

\section{Functional Roles}

We consider DLs with functional roles. A fundamental
observation is that for the basic such DL \ALCF, IQ-rewritability is
undecidable. This can be proved by a reduction from OMQ emptiness in
$(\ALCF,\text{IQ})$~\cite{DBLP:journals/jair/BaaderBL16}.
\begin{restatable}{theorem}{thmUndecALCF}\label{lthmUndecALCF}
  In $(\ALCF,\text{CQ})$, IQ-rewritability is undecidable.  
\end{restatable}
In the following, we show that decidability is regained in the case
where the TBox is empty (apart from functionality assertions).  This
is challenging because functionality assertions have a strong and
subtle impact on rewritability.  As before, the only interesting ABox
signature to be combined with `empty' TBoxes is the full ABox
signature.  We use \Fmc to denote the TBox language in which TBoxes
are sets of functionality assertions and concentrate on rewriting into
IQs that may use inverse roles.
\begin{example}\label{ex:fr}
  Consider the CQ $p(x) = \exists y(s(x,y) \wedge r(y,y))$ from
  Point~1 of Example~\ref{ex:1}.  Then $Q_{s}=(\Tmc_{s},\Sigma_{\sf
    full},p(x))$ and $Q_{r}=(\Tmc_{r},\Sigma_{\sf full},p(x))$ with
  $\Tmc_{w}=\{{\sf func}(w)\}$ for $w\in \{r,s\}$ are both rewritable
  into an OMQ $(\Tmc_{w},\Sigma_{\sf full},q_{w}(x))$ with $q_w(x)$ an
  \ALCI-IQ. The rewritings are neither trivial to find nor entirely
  easy to understand. In fact, for $q_s(x)$ we can use $\forall s.P
  \rightarrow \exists s.(P \rightarrow \exists r.P) $. For $q_r(x)$,
  we introduce three fresh concept names rather than a single one and
  use them in a way inspired by graph colorings:
 $$
 q_{r}(x) = (\forall s.\bigsqcup_{1\leq i \leq 3}P_{i}) \rightarrow (\exists s.(\bigsqcap_{1\leq i \leq 3}(P_{i}\rightarrow \exists r.P_{i})).
 $$

\end{example}
Before giving a characterization of rewritable queries, we introduce
some preliminaries.  Let $q(x)$ be a CQ and $\Tmc$ an
$\mathcal{ALCIF}$-TBox.
A sequence $x_{0},\ldots,x_{n}$ of variables in $q(x)$ is a
\emph{functional path in $q(x)$ from $x_{0}$ to $x_{n}$ w.r.t.~$\Tmc$}
if for all $i<n$ there is a role $r$ such that $\mn{func}(r) \in \Tmc$
and $r(x_{i},x_{i+1})$ is in $q(x)$. We say that $q(x)$ is \emph{f-acyclic
  w.r.t.~$\Tmc$} if for every cycle
$r_{0}(x_{0},x_{1}),\ldots,r_{n-1}(x_{n-1},x_{n})$ in $q(x)$, one of
the following holds:
\begin{itemize}
\item there is a functional path in $q(x)$ from $x$ to some $x_{i}$; 
\item $\mn{func}(r_{i}) \in \Tmc$ or $\mn{func}(r_{i}^{-})
  \in \Tmc$ for all $i<n$ and there is a functional path
  $y_{0},\ldots,y_{m}$ in $q(x)$ with $x_{0}=y_{0}=y_{m}$ such that
  $\{x_{0},\ldots,x_{n-1}\}\subseteq \{y_{0},\ldots,y_{m}\}$.
\end{itemize}
%
We are now ready to state the characterization.
\begin{theorem}
\label{thm:f}
  An OMQ $Q=(\Tmc,\Sigma_{\text{full}},q(x))$ from $(\Fmc,\text{UCQ})$
  is rewritable into an OMQ from $(\Fmc,\ALCI\text{-IQ})$ iff there is a
  subquery $q'(x)$ of $q(x)$ that is
  f-acyclic, connected, and equivalent to $q(x)$.

  When $\mathcal{ALCI}$-IQ is replaced with $\mathcal{ALCI}^{u}$-IQ,
  the same equivalence holds except that
  connectedness is dropped.
\end{theorem}
The proof of Theorem~\ref{thm:f} extends the ultrafilter construction 
from~\cite{DBLP:journals/japll/KikotZ13}. 
We remark that the ``if'' direction in Theorem~\ref{thm:f} even holds
for OMQs $Q=(\Tmc,\Sigma,q(x))$ from
$(\mathcal{ALCIF},\text{UCQ})$. Thus, the case of the `empty' TBox can
again be seen as an underapproximation of the general case. We further
remark that \Tmc remains unchanged in the construction of the
IQ-rewritings 
and that the constructed rewritings are of polynomial size.
\begin{theorem} 
  For OMQs from $(\Fmc,\text{UCQ})$, rewritability into
  $(\Fmc,\ALCI\text{-IQ})$ is {\sc NP}-complete.
\end{theorem}

\section{MMSNP and CSP}
\label{sect:mmsnp}

Recall from the introduction that the OMQ languages studied in this
paper are closely related to CSPs and their logical generalization
MMSNP. In fact, the techniques used to establish the results in
Sections~\ref{sect:charact} and~\ref{sect:compl} can be adapted to
determine the complexity of deciding whether a given MMSNP sentence is
equivalent to a CSP. In a nutshell, we prove that an MMSNP-sentence is
equivalent to a CSP iff it is preserved under disjoint union and
equivalent to a generalized CSP (a CSP with multiple templates), and
that both properties can be reduced to containment between MMSNP
sentences which is 2\NExpTime-complete
\cite{DBLP:conf/kr/BourhisL16}. The latter reduction involves
constructing an MMSNP sentence $\vp_{\mn{acyc}}$ that is reminiscent
of the query $q_{\mn{acyc}}$ in Theorem~\ref{thm:central}. Full
details are given in the appendix.

%
\begin{restatable}{theorem}{thmMMSNPrewr}
\label{thm:MMSNPRewr}
  It is 2\NExpTime-complete to decide whether a given MMSNP-sentence
  is equivalent to a CSP.
\end{restatable}

\section{Conclusion}

We  made a step forward in understanding the relation between
(U)CQs and IQs in ontology-mediated querying. Interesting open
problems include a characterization of IQ-rewritability for DLs with
functional roles when the TBox is non-empty and characterizations for
DLs with transitive roles. The remarks after Theorem~\ref{sect:compl}
and~\ref{thm:PSigComplexity} mention further problems left open.
It would also be worthwhile to follow \cite{DBLP:conf/dlog/KikotTZZ13} and understand the
value of IQ-rewritings for the purposes of efficient practical
implementation.

\smallskip
\noindent
{\bf Acknowledgements.}
Cristina Feier and Carsten Lutz were supported by ERC Consolidator
Grant 647289 CODA. Frank Wolter was supported by EPSRC UK grant
EP/M012646/1.

\cleardoublepage

\appendix
\appendixpage
\addappheadtotoc
\section{Some Technical Preliminaries}

Every interpretation \Imc can be viewed as an undirected graph
$G^u_\Imc$, analogously to the definition of the undirected graph
$G^u_q$ of a CQ $q$. The universal role does not give rise to
edges in~$G^u_\Imc$.


An interpretation is \emph{tree-shaped} or a \emph{tree
  interpretation} if $G^u_\Imc$ is a tree and there are no
multi-edges, that is, $(d,e) \in r^{\Imc}$ implies $(d,e) \notin
s^\Imc$ for all (potentially inverse) roles $s \neq r$.  Let \Tmc be
an $\ALCHI^u$-TBox and \Amc an ABox. An interpretation \Imc is a
\emph{forest model of \Amc} if there are tree interpretations
$(\Imc_a)_{a \in \mn{ind}(\Amc) \cup \Dmc}$, where $\Dmc$ is a
(potentially empty) set of individuals, with mutually disjoint
domains, and 
$$\Delta^{\Imc_a} \cap \mn{ind}(\Amc) =\begin{cases} \{ a
  \}, & \text{ if } a \in \mn{ind}(\Amc) \\
  \emptyset, & \text{ if } a \in \Dmc,
\end{cases}
$$
such that $\Imc$ is the (non-disjoint) union of $\Imc_\Amc$ and
$(\Imc_a)_{a \in \mn{ind}(\Amc) \cup \Dmc}$ where $\Imc_\Amc$ is \Amc
viewed as an interpretation. An \emph{extended forest model} $\Imc$ of
\Amc and \Tmc is a model of \Amc and \Tmc that can be obtained from a
forest model \Jmc of \Amc by closing under role inclusions from
$\Tmc$, that is, adding $(d,e)$ to $r^\Imc$ when $(d,e) \in s^\Jmc$
and $\Tmc \models s \sqsubseteq r \in \Tmc$.  We also say that \Jmc
\emph{underlies} \Imc.

Lemmas of the following kind have been widely used in the literature
on ontology-mediated querying. The proof of the ``if'' direction uses
a standard unraveling argument and is omitted, see for example
\cite{DBLP:conf/cade/Lutz08}. 
\begin{lemma}
  \label{lem:canmod}
  Let $Q=(\Tmc,\Sigma,q(\xbf)$ be an OMQ from
  $(\ALCHI^u,\text{UCQ})$, \Amc a $\Sigma$-ABox, and
  $\abf \subseteq \mn{ind}(\Amc)$. Then $\Amc \models Q(\abf)$ iff for
  all extended forest models \Imc of \Amc and \Tmc, $\Imc \models Q(\abf)$.
\end{lemma}
We introduce some more helping lemmas. 
An ABox \Amc can be seen as a directed graph $G_\Amc$ and as an undirected graph $G^u_\Amc$ in the expected way, analogously to the definition of $G_q$ and $G^u_q$ for a CQ $q$.  For an ABox \Amc and $a \in \mn{ind}(\Amc)$, we use $\Amc^{\mn{con}}_{a}$ to denote the restriction of \Amc to the
individuals reachable in $G^u_\Amc$ from $a$.  We also denote with $\mn{CON}_{\Amc}$ the set of ABoxes induced by the maximal connected components of $G^u_\Amc$.
%

\begin{lemma}
\label{lem:conIQ}
Let $Q=(\Tmc,\Sigma,q(x))$ be an OMQ from $(\ALCHI, \text{IQ})$. Then $\Amc \models Q(a)$ implies $\Amc^{\mn{con}}_a \models Q(a)$.  
\end{lemma}
A \emph{homomorphism} from an ABox \Amc to an ABox \Bmc is a function
$h: \mn{ind}(\Amc) \rightarrow \mn{ind}(\Bmc)$ such that $A(a) \in
\Amc$ implies $A(h(a)) \in \Bmc$ and $r(a,b) \in \Amc$ implies
$r(h(a),h(b)) \in \Bmc$. We write $\Amc \rightarrow \Bmc$ to indicate
that there is a homomorphism from \Amc to \Bmc. For $a \in
\mn{ind}(\Amc)$ and $b \in \mn{ind}(\Bmc)$, we further write $(\Amc,a)
\rightarrow (\Bmc,b)$ to indicate that there is a homomorphism $h$
from \Amc to \Bmc with $h(a)=b$. The following lemma is well-known,
see for example \cite{DBLP:journals/tods/BienvenuCLW14}. 


\begin{lemma}
\label{lem:hombasic}
  Let $Q=(\Tmc,\Sigma,q)$ be a unary OMQ from $(\ALCHI^u,\Qmc)$, with $\Qmc \in \{\text{UCQ}, \text{IQ}\}$,  $\Amc$ and $\Bmc$ be $\Sigma$-ABoxes,
  $a \in \mn{ind}(\Amc)$, and $b \in \mn{ind}(\Bmc)$. Then $(\Amc,a) \rightarrow (\Bmc,b)$ and $\Amc \models Q(a)$ implies $\Bmc \models Q(b)$.
\end{lemma}





\section{Proofs for Section 3}

We start with introducing several lemmas concerned with certain
constructions on ABoxes. These lemmas are closely related to the
connection between ontology-mediated querying and constraint
satisfaction problems (CSPs), see for example
\cite{DBLP:journals/tods/BienvenuCLW14,DBLP:journals/lmcs/LutzW17}.

Note that in assertions $r(x,y)$ in an ABox, $r$ must be a role name
but cannot be an inverse role. For purposes of uniformity, we use
$r^-(x,y)$ as an alternative notation to denote an assertion $r(y,x)$
in an ABox. A \emph{cycle} in an ABox is defined exactly like a cycle
in a CQ, repeated here for convenience. A \emph{cycle} in an ABox \Amc
is a sequence of non-identical assertions
$r_0(a_0,a_1),\dots,r_{n-1}(a_{n-1},a_n)$ in~\Amc, $n \geq 1$, where
 \begin{enumerate}

 \item $r_0,\dots,r_{n-1}$ are (potentially inverse) roles,

 \item $a_i \neq a_j$ for $0 \leq i < j < n$, and $a_0=a_n$.

 \end{enumerate}
 The \emph{length} of this cycle is $n$.  The \emph{girth} of 
 $\Amc$ is the length of the shortest cycle in it and $\infty$ if
 $\Amc$ has no cycle.
%
%
%
%
%

The following is a DL formulation of what is often known as the sparse
incomparability lemma in CSP \cite{DBLP:journals/siamcomp/FederV98}.
\begin{lemma} \label{lem:sparse} For every ABox \Amc and all $g,s \geq
  0$, there is an ABox $\Amc^g$ of girth exceeding $g$ such that 
  \begin{enumerate}

  \item $\Amc^g \rightarrow \Amc$ and

  \item for every ABox \Bmc with $|\mn{ind}(\Bmc)| \leq s$, $\Amc
    \rightarrow \Bmc$ iff $\Amc^g \rightarrow \Bmc$.
  \end{enumerate}
\end{lemma}

We next establish a `pointed' version of Lemma~\ref{lem:sparse} that
is crucial for the subsequent proofs. The \emph{$a$-girth of \Amc} is
defined exactly like the girth except that we only consider cycles
that \emph{do no pass through~$a$}.
\begin{lemma}
\label{lem:sparsepointed} 
For all ABoxes \Amc, $a \in \mn{ind}(\Amc)$, and $g,s \geq 0$, there is an 
ABox $\Amc^g$ of $a$-girth exceeding $g$ such that 
\begin{enumerate}

\item $(\Amc^g, a) \rightarrow (\Amc, a)$

\item for every ABox $\Bmc$ with $|\mn{ind}(\Bmc)| \leq s$ and every $b \in
  \mn{ind}(\Bmc)$, $(\Amc,a) \rightarrow (\Bmc,b)$ iff $(\Amc^g,a) \rightarrow (\Bmc,b)$.

\end{enumerate}
\end{lemma}
\noindent
\begin{proof}\ Let \Amc be an ABox, $a \in \mn{ind}(\Amc)$, and $g,s
  \geq 0$.  Further, let $\Amc_+$ be the ABox obtained from \Amc by
  adding the assertion $P(a)$, $P$ a fresh concept name, let
  $\Amc^g_+$ the ABox obtained from $\Amc_+$ by applying
  Lemma~\ref{lem:sparse} for $g$ and $s$, and let $h$ be a
  homomorphism from $\Amc^g_+$ to $\Amc_+$. Assume w.l.o.g.\ that the
  individual name $a$ does not occur in $\Amc^g_+$. We use $\Amc^g$ to
  denote the ABox obtained from $\Amc^g_+$ by dropping all facts of
  the form $P(b)$ and identifying all individual names $b$ with
  $h(b)=a$, replacing them with $a$.  We show that $\Amc^g$ is as
  required:
\begin{itemize}

\item[(a)] $\Amc^g$ has $a$-girth higher than $g$.

  Every cycle in $\Amc^g$
  that does not pass through $a$ is also in $\Amc^g_+$, thus is of
  length exceeding $g$.

\item[(b)] Point~1 of Lemma~\ref{lem:sparsepointed}  is satisfied. 

  Let $h':\Ind(\Amc^g)\to \Ind(\Amc)$ be such that $h'(a)=a$ and
  $h'(b)=h(b)$ if $a \neq b$. It can be verified that $h'$ is a
  homomorphism from $\Amc^g$ to~\Amc. It clearly witnesses $(\Amc^g,
  a) \rightarrow (\Amc, a)$, as required.

\item[(c)] Point~2 of Lemma~\ref{lem:sparsepointed}  is satisfied. 

  Let $\Bmc$ be an ABox with $|\mn{ind}(\Bmc)| \leq s$ and $b \in
  \mn{ind}(\Bmc)$. We have to show that $(\Amc,a) \rightarrow
  (\Bmc,b)$ iff $(\Amc^g,a) \rightarrow (\Bmc,b)$. The ``only if''
  direction is immediate by (b). For the ``if'' direction, assume that
  $(\Amc^g, a) \rightarrow (\Bmc, b)$. Then $\Amc^g \cup \{P(a)\}
  \rightarrow \Bmc \cup \{P(b)\}$. This implies $\Amc^g_+ \rightarrow
  \Bmc \cup \{P(b)\}$ and by Point~2 of Lemma~\ref{lem:sparse} also
  $\Amc_+ \rightarrow \Bmc \cup \{P(b)\}$. As $P(a)$ is the only
  assertion of this form in $\Amc_+$, it follows that $(\Amc_+, a)
  \rightarrow (\Bmc \cup \{P(b)\}, b)$, thus $(\Amc, a) \rightarrow
  (\Bmc, b)$. 

\end{itemize}
\end{proof}

The following lemma is a straightforward variation of similar lemmas
from \cite{DBLP:journals/tods/BienvenuCLW14}. The constructed ABoxes
are called \emph{CSP templates} in \cite{DBLP:journals/tods/BienvenuCLW14}.
\begin{lemma}\label{lem:ToBeDone}
  Let $Q=(\Tmc,\Sigma,q(x))$ be an OMQ from $(\ALCHI^u,\text{IQ})$. Then
  one can find a set $\Gamma$ of pairs $(\Bmc,b)$ with $\Bmc$ a
  $\Sigma$-ABox and $b \in \mn{ind}(\Bmc)$ such that for every
  $\Sigma$-ABox $\Amc$ and all $a \in \mn{ind}(\Amc)$,
  \begin{enumerate}

  \item $\Amc \models Q(a)$ iff $(\Amc, a) \not \rightarrow (\Bmc,b)$
    for all $(\Bmc,b) \in \Gamma$;

\item $\Amc$ is consistent with $\Tmc$ iff $\Amc \rightarrow \Bmc$
  for some $(\Bmc,b) \in \Gamma$.

  \end{enumerate}
  When $Q$ is from $(\ALCHI,\text{IQ})$, then $\Gamma$ can be chosen
  so that all ABoxes in it are identical.
\end{lemma}

An \emph{$\ELI^u$-concept} is an \ALCI-concept that uses only the
constructors $\sqcap$, $\exists r . C$, $\exists r^- . C$, and
$\exists u . C$ where $u$ is the universal role.

\lemrewrALCI*

\noindent
\begin{proof}\ 
  %
  For Point~1, let $Q=(\tbox, \Sigma, q(x))$ be an OMQ from
  $(\ALCHI^u, \text{UCQ})$ where $q(x)$ is $x$-acyclic and
  connected. Further, let $Q'=(\Tmc,\Sigma, (P \rightarrow \bigsqcup_{p(x) \text{ a CQ in } q(x)} C_p)(x))$ be as constructed
  in the main part of the paper. We have to show the following:
%
\begin{itemize}
\item ``$Q \subseteq Q'$'': Let \Amc be a $\Sigma$-ABox with $\Amc
  \models Q(a)$. Then, for every model $\Imc$ of $\Tmc$ and $\Amc$, $\Imc
  \models p(a)$, for some CQ $p(x)$ in $q(x)$, and thus $\Imc \models \neg P(a)$ or $\Imc \models P(a) \wedge p(a)$.  As $\Amc_{p''} \to \Amc_p \cup \{P(x)\}$, the latter is the same as $\Imc \models p''(a)$ or $\Imc \models C_p(a)$. Thus $\Imc \models (P \rightarrow C_p)(a)$, or $\Imc%
  \models Q'(a)$.

\item ``$Q' \subseteq Q$'': Let \Amc be a $\Sigma$-ABox with $\Amc
  \models Q'(a)$. Then, for every model $\Imc=(\Delta, \cdot^{\Imc})$ of $\Tmc$ and $\Amc$,
  $\Imc \models (P \rightarrow \bigsqcup_{p(x) \text{ a CQ in } q(x)} C_p)(a)$. Then, there must be a model
  $\Imc'=(\Delta, \cdot^{\Imc'})$ of $\Tmc$ and $\Amc$ (possibly the same as \Imc) such that $P^{\Imc'}=\{a^{\Imc'}\}$ and $\cdot^{\Imc}$ and $\cdot^{\Imc'}$ coincide on all other symbols. We have $\Imc' \models (P \rightarrow \bigsqcup_{p(x) \text{ a CQ in } q(x)} C_p)(a)$, thus $\Imc' \models C_p(a)$, for
  some CQ $p(x)$ in $q(x)$. From the construction of $C_p$: $\Imc' \models 
  p''(a)$, and from the fact that $P$ is interpreted as a singleton in
  $\Imc'$, $\Imc' \models p(a)$, or $\Imc' \models \bigvee_{p(x) \text{ a CQ in } q(x)} 
  p(a)$. As $P$ is fresh, and in particular does not occur in $q$,
  and $\Imc$ and $\Imc'$ might differ only w.r.t. the interpretation of $P$:
  $\Imc \models \bigvee_{p(x) \text{ a CQ in } q(x)} p(a)$. Thus, $ \Amc \models
  Q(a)$.  

\end{itemize}

For Point~2, let $Q=(\tbox, \Sigma, q(x))$ be an OMQ from $(\ALCHI^u,
\text{UCQ})$ where $q(x)$ is $x$-acyclic.  Let
$p_{0}(x),p_{1}(),\dots,p_{m}()$ be the maximal connected components
of $q(x)$. Note that $p_0(x)$ is $x$-acyclic and each $p_i()$ is
acyclic in the sense that it contains no cycles at all. We can view
$p_0(x)$ as an \ELI-concept and each $p_i()$ as an $\ELI^u$-concept
$C_{p_i}$ of the form $\exists u . C$ with $u$ the universal role and
$C$ an \ELI-concept. Let $C_p= C_{p, 0} \sqcap C_{p,1} \sqcap \ldots
\sqcap C_{p,m}$ and $Q'=(\Tmc,\Sigma, (P \rightarrow \bigsqcup_{p(x)
  \text{ a CQ in } q(x)} C_p)(x))$. One can show that $Q \equiv Q'$.
%
\end{proof}

\begin{example}
\label{ex:ALCIrewr}
Let $Q$ be an OMQ $(\Tmc, \Sigma, q(x))$ with $\Tmc=\emptyset$, $\Sigma=\{r, s, t, v\}$, and  $q(x)=\exists y_1 \exists y_2 \exists y_3 \, r(x,y_1) \wedge s(x, y_2) \wedge t(y_2, y_1) \wedge v(y_2, y_3)$. It is easy to see that $q(x)$ is $x$-acyclic and connected. Towards obtaining an $\ALCI$-IQ rewriting, we construct a tree-shaped CQ $p''(x)$ from $q(x)$ by first removing the atom $s^-(y_2, x)$  and then adding atoms $s^-(y_2, x')$ and $P(x')$, with $x'$ a fresh variable: $p''(x)=\exists y_1 \exists y_2 \exists y_3 \exists x'\, P(x') \wedge r(x,y_1) \wedge s^-(y_2, x')\wedge t(y_2, y_1)\wedge v(y_2, y_3)$. 
The concept $C_q$ corresponding to $p''(x)$ is $\exists r. \exists t^-.(\exists v. \top\sqcap \exists s^-. P)$, and thus the desired rewriting is the OMQ $Q'=(\Tmc, \Sigma, q'(x))$ with $q'(x)$ the \ALCI-IQ: $(P \rightarrow C_q)(x)$. 
\end{example}

\thmcentral*

\noindent
\begin{proof}\ The implication ``$2 \Rightarrow 1$'' is trivial.  For
  ``$3 \Rightarrow 2$", assume that $Q \equiv (\Tmc, \Sigma,
  q^{\mn{con}}_{\mn{acyc}}(x))$. Since $q^{\mn{con}}_{\mn{acyc}}$ is connected and
  $x$-acyclic, we can apply Lemma~\ref{lem:rewrAlgALCI}.
  

\smallskip

For ``$1 \Rightarrow 3$'', we show that whenever an OMQ $Q$ from $(\Lmc,\text{UCQ})$ is IQ-rewritable, then (a)~$Q \equiv
Q_{\mn{acyc}}$ where $Q_{\mn{acyc}} = (\Tmc, \Sigma,q_{\mn{acyc}}(x))$ and (b)~$Q \equiv Q^{\mn{con}}$ where $Q^{\mn{con}}:=(\Tmc,\Sigma,q^{\mn{con}}(x))$. This yields $Q \equiv
(\Tmc,\Sigma,q^{\mn{con}}_{\mn{acyc}}(x))$ as desired: if $Q$ is
IQ-rewritable, then (a) yields $Q \equiv Q_{\mn{acyc}}$, thus 
$Q_{\mn{acyc}}$ is IQ-rewritable and we can apply (b). 

Thus, let $Q$ from $(\Lmc,\text{UCQ})$ be IQ-rewritable. Thus there is
an OMQ $Q'=(\Tmc', \Sigma, C(x))$ from $(\Lmc,\text{IQ})$ that is
equivalent to~$Q$. By Lemma~\ref{lem:ToBeDone}, one can find a
$\Sigma$-ABox \Bmc and $b_1,\dots,b_k \in \mn{ind}(\Bmc)$ such that
for every $\Sigma$-ABox $\abox$ and $a \in \mn{ind}(\abox)$,
  \begin{enumerate}

  \item
$\Amc \models Q'(a)$ iff $(\abox, a) \not \rightarrow (\Bmc,b_i)$ for \mbox{$1 \leq i \leq k$};

\item $\Amc$ is consistent with $\Tmc \cup \Tmc'$ iff $\abox \rightarrow \Bmc$.

  \end{enumerate}
  We show Points~(a) and~(b) from above.

  \medskip

(a) We have $Q_{\mn{acyc}} \subseteq Q$ by definition of $Q_{\mn{acyc}}$, no matter whether $Q$ is IQ-rewritable or not, and thus it remains to show that $Q \subseteq Q_{\mn{acyc}}$.  If all CQs in $q(x)$ are $x$-acyclic, the result clearly holds. In the following we assume that at least one CQ in $q(x)$ is not $x$-acyclic.

  Let \Amc be a $\Sigma$-ABox with $\Amc \models Q(a)$. Thus $(\Amc,
  a) \not \rightarrow (\Bmc,b_i)$ for $1 \leq i \leq k$.  We apply
  Lemma~\ref{lem:sparsepointed} with $g$ the maximum between $2$ and the girths of CQs from $q(x)$ which are not $x$-acyclic, obtaining a $\Sigma$-ABox $\Amc^g$ of $a$-girth exceeding  $g$ such that $(\Amc^g, a) \rightarrow (\Amc, a)$ and $(\Amc^g, a) \not \rightarrow (\Bmc,b_i)$ for $1 \leq i \leq k$. The latter yields $\Amc^g \models Q(a)$. We aim to show that $\Amc^g \models Q_{\mn{acyc}}(a)$. Since $(\Amc^g, a) \rightarrow (\Amc, a)$, it follows by Lemma~\ref{lem:hombasic} that $\Amc \models Q_{\mn{acyc}}(a)$, as desired.
  
By Lemma~\ref{lem:canmod}, it suffices to show that for every extended forest model \Imc of $\Amc^g$ and $\Tmc$, we have $\Imc \models q_{\mn{acyc}}(a)$. Thus let \Imc be such a model. Since $\Amc^g \models Q(a)$, we have $\Imc \models q(a)$ and thus there is a CQ $p(x)$ in $q(x)$ such that $\Imc \models p(a)$.  Consequently, there is a homomorphism $h$ from $p(x)$ to \Imc with $h(x)=a$. Let $p'(x)$ be the contraction of $p(x)$ obtained by identifying all variables $y_1$ and $y_2$ such that $h(y_1)=h(y_2)$. As witnessed by $h$, $\Imc \models p'(a)$. Note that the $x$-girth of $p'(x)$ is either $\infty$ or it is bounded from above by $g$ since the $x$-girth of $p(x)$ is. Also note that $h$ is an injective homomorphism from $p'(x)$ to \Imc.  By definition of extended forest models, all cycles in $\Imc$ are either cyles from $\Amc^g$, or they are cycles of the form  $r(y, z), s(z, y)$.  This together with the fact that the girth of $\Amc^g$ exceeds $g$ implies that every cycle in $p'(x)$ passes through $x$ or is of the latter kind. In fact, $p'(x)$ is $x$-acyclic when \Tmc contains no role inclusions since then \Imc is a forest model of $\Amc^g$. Since $p'(x)$ is a CQ in $q_{\mn{acyc}}(x)$, we are done in that case.

Now for the case where \Tmc contains role inclusions. Let $\Jmc$ be the forest model of $\Amc^g$ underlying $\Imc$. Construct a CQ $p''(x)$ from $p'(x)$ as follows: for all distinct variables $y,z$, with $y \neq x$ and $z \neq x$, whenever $r_1(y,z),$ $\dots,$ $r_k(y,z),$ $s_1(z,y),\dots,s_\ell(z,y)$ are  all atoms of this form in $p'(x)$, then replace them with $r(x,y)$ if $(h(x),h(y)) \in r^\Jmc$ and with $r(y,x)$ if $(h(y),h(x)) \in
  r^\Jmc$. Note that by definition of extended forest models and due to the fact that $g$, the girth of $\Amc^g$, is greater than $2$, such an $r$ always exists. As witnessed by~$h$, $\Imc \models p''(a)$. Moreover, $p''(x)$ is $x$-acyclic and a CQ in $q_{\mn{acyc}}(x)$, thus we are again done.
  
\medskip


(b) It is immediate by definition of $Q^{\mn{con}}$ that
$Q \subseteq Q^{\mn{con}}$.  We thus have to show that
$Q^{\mn{con}}\subseteq Q$.  Assume the contrary. Then,
there is a $\Sigma$-ABox \Amc and an $a \in \mn{ind}(\Amc)$ such that
$\Amc \models Q^{\mn{con}}(a)$ and $\Amc \not \models
Q(a)$. Note that $\Amc$ must be consistent with~\Tmc. Since $Q$ is
non-empty, there is a $\Sigma$-ABox $\Amc_Q$ such that $\Amc_Q \models
Q(b)$ for some $b \in \mn{ind}(\Amc_Q)$. Let $\Amc'$ be the disjoint
union of \Amc and $\Amc_Q$.  We get $\Amc' \models Q(a)$ from
Lemma~\ref{lem:hombasic} and thus $\Amc' \models Q'(a)$.  Since $Q'$
is from $(\ALCI,\text{IQ})$ and ${\Amc'}^{\mn{con}}_a = \Amc$, the
latter and Lemma~\ref{lem:conIQ} implies $\Amc \models Q'(a)$, thus
$\Amc \models Q(a)$, a contradiction.

\medskip
 

When the OMQ language $\Lmc$ is replaced by $\Lmc^u$, we can show that
$Q \equiv Q_{\mn{acyc}}$ exactly as above. The second part of the proof
showing that $Q \equiv Q^{\mn{con}}$ (does not go through
and) is no longer needed.
\end{proof}

\begin{lemma}
\label{lem:rewrAlgALCTBox}
Let $Q=(\tbox, \Sigma, q(x))$ be an OMQ from $(\ALCH, \text{UCQ})$
such that $q(x)$ is $x$-acyclic and connected. Then $Q$ is rewritable
into an OMQ $(\Tmc \cup \Tmc',\Sigma,q(x))$ from $(\ALCH,\text{IQ})$
whose size is polynomial in the size of $Q$.
\end{lemma}

\noindent
\begin{proof}\ Let $Q=(\tbox, \Sigma, q(x))$ be an OMQ from $(\ALCH,
  \text{UCQ})$ such that $q(x)$ is $x$-acyclic and connected.  From
  Lemma~\ref{lem:rewrAlgALCI}, we know that there is an OMQ
  $Q'=(\tbox, \Sigma, C(x))$ that is equivalent to $Q$, with
  $C(x)$ an $\ALCI$-IQ. From the proof of the lemma, we further know
  that $C$ has the form $P \rightarrow \bigsqcup_{p(x) \text{ a CQ in
    } q(x)} C_p$ where each $C_p$ is an $\ELI$-concept. We show how to
  transform $Q'$ into an equivalent OMQ $(\tbox', \Sigma, C'(x))$ from
  $(\ALCH, \text{IQ})$.

  We start with setting $\tbox':=\Tmc$ and $C':=C$ and apply the
  following modification step until no further changes are possible:
  if $D$ is a subconcept of $C'$ that is of the form $\exists r^-.E$ with
  $E$ an \EL-concept, then let $P_D$ be a fresh concept name that
  is not in $\Sigma$ and
\begin{itemize}
\item set $\Tmc'=\Tmc'\cup\{E \sqsubseteq \forall r. P_D\}$ and
\item replace $\exists r^-.E$ in $C'$ with $P_D$. 
\end{itemize}
At the end of the transformation, $C'$ will contain no inverse roles
anymore, so the constructed OMQ is from $(\ALCH,
\text{IQ})$. Moreover, it is straightforward to show that the
described modification step preserves equivalence of the OMQ.

In fact, assume that $Q_2=(\Tmc_2,\Sigma,C_2(x))$ was obtained by a
single modification step from $Q_1=(\Tmc_1,\Sigma,C_1(x))$. Let $\Amc$
be a $\Sigma$-ABox and $a \in \mn{Ind}(\Amc)$. First assume that $\Amc
\not\models Q_1(a)$. Then there is a model \Imc of \Amc and $\Tmc_1$
with $a \notin C_1^\Imc$.  Extend \Imc to the concept name $P_D$
by setting $P_D^\Imc = (\exists r^-.E)^\Imc$. Clearly, \Imc is then a
model of $\Tmc_2$. Moreover, by construction of $C_2$ we have $a
\notin C_2^\Imc$. Conversely, assume that $\Amc \not\models Q_1(a)$.
Then there is a model \Imc of \Amc and $\Tmc_2$ with $a \notin
C_2^\Imc$. Clearly, \Imc is a model of $\Tmc_1$.  Since \Imc is a
model of $\Tmc_2$, we have $(\exists r^-.E)^\Imc \subseteq
P_D^\Imc$. We can modify \Imc by setting $P_D^\Imc = (\exists
r^-.E)^\Imc$ and the resulting \Imc will still be a model of $\Tmc_1$
and still satify $a \notin C_2^\Imc$ since all occurrences of $P_D$ in
$C_2$ are positive. Moreover, by construction of $C_2$ it also
satisfies $a \notin C_1^\Imc$.
\end{proof}

The proof of the following lemma is a much simplified and slightly
extended version of a construction from
\cite{DBLP:journals/japll/KikotZ13}.

\begin{lemma}
\label{lem:rewrAlgALC}~\\[-4mm]
%
\begin{enumerate}
\item Every OMQ $Q=(\tbox, \Sigma, q(x))$ from $(\ALCH, \text{UeCQ})$
  with $q(x)$ $x$-acyclic and $x$-accessible is rewritable into an OMQ
  $Q=(\tbox, \Sigma, C(x))$ with $C(x)$ an \ALC-IQ and

\item Every OMQ $Q=(\tbox, \Sigma, q(x))$ from $(\ALCH^u,
  \text{UeCQ})$ with $q(x)$ $x$-acyclic is rewritable into an OMQ
  $Q=(\tbox, \Sigma, C(x))$ with $C(x)$ an $\ALC^u$-IQ.
%
\end{enumerate}
The size of the IQs $C(x)$ is polynomial in the size of $q(x)$. 
\end{lemma}
%

\noindent
\begin{proof} \ We first observe that Lemma~\ref{lem:rewrAlgALCI}
  extends to the case where the actual query is a UeCQ rather than a
  UCQ. One simply ``carries through'' atoms $C(x)$ with $C$ a compound
  concept in the construction of the IQ.

  We start with Point~2 since its proof is simpler and prepares for
  the proof of Point~1. Thus, let $Q=(\tbox, \Sigma, q(x))$ be an OMQ
  from $(\ALCH^u, \text{UeCQ})$ with $q(x)$ $x$-acyclic.  By (the
  extended) Lemma~\ref{lem:rewrAlgALCI}, there is an equivalent OMQ
  $Q'=(\Tmc, \Sigma, C(x))$ with $C(x)$ an $\ALCI^u$-IQ.  In fact, the
  IQ $C(x)$ constructed in the proof of Lemma~\ref{lem:rewrAlgALCI} is
  of the form
  $P \rightarrow \bigsqcup_{p(x) \text { a CQ in } q(x)} C_p(x)$ where
  each $C_p$ is an \emph{$\ELI^u$-concept decorated with
    $\ALC^u$-concepts}, that is, built according to the syntax rule
$$
  C,D ::= \top \mid A \mid C \sqcap D \mid \exists r . D \mid \exists
  u . D \mid E
$$
where $A$ ranges over all concept names, $r$ over all (potentially
inverse) roles, and $E$ over all $\ALC^u$-concepts. Note that every
$\ELI^u$-concept decorated with $\ALC^u$-concepts is an
$\ALCI^u$-concept, but that the converse is false.

We construct from $Q'$ an $(\ALCH^u, \text{IQ})$-rewriting $(\Tmc,
\Sigma, C'(x))$ of $Q$ where $C'$ has the form $C_{\mn{pre}}
\rightarrow C_{\mn{con}}$. To start, let $D_1=\exists
r_1^-.P,\dots,D_\ell = \exists
r_\ell^-.P$ be all subconcepts of $C$ that are of this form and let
\begin{itemize}

\item $C_{\mn{con}}$ be obtained from $C$ by replacing each
  concept $D_i$ with a fresh concept name $P_{D_i} \notin \Sigma$ and

\item $C_{\mn{pre}}=\forall r_1 . P_{D_1} \sqcap \cdots \sqcap \forall
  r_\ell . P_{D_\ell}$.

\end{itemize}
Next, exhaustively apply the following transformation step: if
$D=\exists r^-.E$ is a subconcept of $C_{\mn{con}}$ where $E$ is an
$\ALC^u$-concept (that is, does not contain any inverse roles), then
\begin{itemize}
\item replace $D$ in $C_{\mn{con}}$ with a fresh concept name $P_D
  \notin \Sigma$ and

\item set
$C_{\mn{pre}} = C_{\mn{pre}} \sqcap \forall u. (E \rightarrow \forall
r. P_D)$.

\end{itemize}
We end up with $C_{\mn{con}}$ being an $\ALC^u$-concept because if
there is a subconcept $\exists r^-.E$ of $C_{\mn{con}}$ left, then in
the innermost such subconcept $E$ must be an $\ALC^u$-concept and thus
the transformation rule applies. It can be proved that the initial IQ
$C_{\mn{pre}} \rightarrow C_{\mn{con}}(x)$ is equivalent to $C(x)$ and
that the transformation step is equivalence preserving. We omit
details, please see the proof of Lemma~\ref{lem:rewrAlgALCTBox} for
very similar arguments.

\medskip

We now turn to Point~1. Let $Q=(\tbox, \Sigma, q(x))$ be an OMQ from
$(\ALCH, \text{UeCQ})$ with $q(x)$ $x$-acyclic and $x$-accessible.
Then $q(x)$ is also connected.  By (the extended)
Lemma~\ref{lem:rewrAlgALCI}, there is an equivalent OMQ $Q'=(\Tmc,
\Sigma, C(x))$ with $C(x)$ an $\ALCI$-IQ.  In fact, the IQ $C(x)$
constructed in the proof of Lemma~\ref{lem:rewrAlgALCI} is of the form
$P \rightarrow \bigsqcup_{p(x) \text { a CQ in } q(x)} C_p(x)$ where
each $C_p$ is an \emph{$\ELI$-concept decorated with $\ALC$-concepts},
that is, an $\ELI^u$-concept decorated with $\ALC^u$-concepts that
does not mention the universal role.
%
However, the syntactic structure of $C$ is even more restricted.
%


\begin{claim}
  In each subconcept $\exists r^-.D $ of $C$, $D=P$ or $D$ has the
  form
  $D_0 \sqcap \exists r_1^{-}.(D_1 \sqcap \exists r_2^{-}.(\ldots
  \sqcap \exists r_n^-.P)) \ldots )$, $n \geq 1$.
\end{claim} 

\noindent
\emph{Proof of claim}. Let $p(x)$ be a CQ in $q(x)$. Recall that, when
constructing $C(x)$ in the proof of Lemma~\ref{lem:rewrAlgALCI}, we
first remove atoms of the form $r(x,y)$ from $p(x)$ to obtain a
tree-shaped CQ $p'(x)$, then add back $r^-(y,u)$ and $P(u)$ for each
removed $r(x,y)$ where $u$ is a fresh variable producing a CQ
$p''(x)$, and finally view $p''(x)$ as an \ELI-concept $C_p$ decorated
with \ALC-concepts.\footnote{The first two steps can together be
  viewed as an unfolding construction.}

Let $\exists r^-.D$ be a subconcept of $C_p$. Then there is a variable
$y$ in $p''(x)$ and an atom $r^-(y,z)$ such that $D$ describes the
subtree of $p''(x)$ rooted at $z$ and $z$ is a successor of $z$ in the
tree-shaped $p''(x)$, that is, $y$ is on the path from the root $x$ of
$p''(x)$ to $z$. First assume that $r^-(y,z)$ was one of the atoms
added back in the construction of $p''(x)$. Then $D=P$ and we are
done. Now assume that $r^-(y,z)$ was already in $p'(x)$ and thus in
$p(x)$. Since $p(x)$ is $x$-accessible, $z$ is reachable from $x$ in
the directed graph $G_p$. Since $z$ is not reachable from $x$ in the
directed graph $G_{p''}$ it follows from the construction of $p'(x)$
and $p''(x)$ that $z$ is reachable in $G_{p''}$ from a leaf node
labeled with $P$. Consequently, $D$ must have the stated form. This
finishes the proof of the claim.

\medskip

We construct from $Q'$ an $(\ALCH, \text{IQ})$-rewriting
$(\Tmc, \Sigma, C'(x))$ of $Q$ where $C'$ has the form
$C_{\mn{pre}} \rightarrow C_{\mn{con}}$.  To start, let
$D_1=\exists r_1^-.P,\dots,D_\ell=\exists r_\ell^-.P$ be all subconcepts of
$C$ that are of this form and let
\begin{itemize}

\item $C_{\mn{con}}$ be obtained from $C$ by replacing each
  concept $D_i$ with a fresh concept name $P_{D_i}$ and

\item $C_{\mn{pre}}=\forall r_1 . P_{D_1} \sqcap \cdots \sqcap \forall
  r_\ell . P_{D_\ell}$.

\end{itemize}
It is easy to see that the following condition is satisfied:
\begin{itemize}

\item[($*$)] in every subconcept $D=\exists r^-.E$ of $C_{\mn{con}}$
  with $E$ an \ALC-concept, $E$ is of the form $F \sqcap P_{D'}$.

\end{itemize}

Next, exhaustively apply the following transformation step, which
preserves~($*$): if $D=\exists r^-.(F \sqcap P_{D'})$ is a subconcept
of $C_{\mn{con}}$ where $F$ is an $\ALC$-concept, then 
\begin{itemize}

\item replace $D$ in
$C_{\mn{con}}$ with a fresh concept name $P_D$ and 

\item replace $P_{D'}$ in
  $C_{\mn{pre}}$ with $F \rightarrow \forall R.P_D$.

\end{itemize}
It can be verified that, because of the claim, the transformation step
indeed preserves ($*$). It can also be seen that all subconcepts of
the form $\exists r^- . E$ will eventually be eliminated. Finally, it
can be shown that the initial IQ $C_{\mn{pre}} \rightarrow
C_{\mn{con}}(x)$ is equivalent to $C(x)$ and that the transformation
step is equivalence preserving. We omit details.
\end{proof}

\begin{example}
\label{ex:ALCrewr}
Let $Q$ be the OMQ from Example~\ref{ex:ALCIrewr}. Towards obtaining an $\ALC$-IQ rewriting, we start with the $\ALCI$-IQ rewriting $Q'$ described in the same example. The only subconcept of the form $\exists r_i^-.P$ in $C_q$ is $D=\exists s^-.P$. We thus introduce a fresh concept name $P_D$ and initialize $C_{\mn{pre}}$ and $C_{\mn{con}}$ with $\forall s. P_D$ and $\exists r. \exists t^-. (\exists v.\top \sqcap  P_{D})$. We next consider concepts of the form $\exists r^-.(F \sqcap P_{D'})$, with $F$ an $\ALC$ concept and $P_{D'}$ previously introduced. The only such concept is $E=\exists t^-. (\exists v. \top \sqcap  P_{D})$. We replace $E$ in $C_\mn{con}$ with $P_E$ and $P_D$ in $C_\mn{pre}$ with $\neg \exists v. \top \sqcup \forall t. P_E$. At this point both $C_\mn{pre}=\forall s. (\forall t. (\neg \exists v. \top \sqcup P_E))$ and $C_{\mn{con}}=\exists r. P_E$ are \ALC concepts, thus no further transformation is possible (and neither needed): $Q$ can can be rewritten into an OMQ $Q''=(\Tmc, \Sigma, C'(x))$ with $C'$ the $\ALC$ concept $\forall s. (\forall t. (\neg \exists v. \top \sqcup P_E)) \rightarrow \exists r. P_E$.
\end{example}

\thmalc*

\noindent
\begin{proof}\
``$2 \Rightarrow 1$'' and ``$1 \Rightarrow 3$'' are trivial. 
\smallskip

``$3 \Rightarrow 2$''. We know from Theorem~\ref{thm:central} that
$(\LImc, \text{IQ})$-rewritability of $Q$ implies that $q(x)$ is
$x$-acyclic and connected. By Lemma
\ref{lem:rewrAlgALCTBox}, $Q$ is rewritable into an OMQ from $(\Lmc,
\text{IQ})$ that is of the desired shape.

\smallskip
``$4 \Rightarrow 5$''.  Let $Q=(\Tmc,\Sigma_{\mn{full}},q(x))$ be an
OMQ from $(\mathcal{ALCH},\text{UCQ})$ and assume that $Q$ is
rewritable into an OMQ $Q'=(\Tmc,\Sigma_{\mn{full}},C(x))$ with $C(x)$
an \ALC-IQ. Let $Q_{\mn{acyc}} = (\Tmc, \Sigma_{\mn{full}},
q^{\mn{con}}_{\mn{acyc}}(x))$. It is established in the proof of the
``$1 \Rightarrow 3$'' direction of Theorem~\ref{thm:central} that,
since $Q$ is rewritable into $(\ALCHI,\text{IQ})$, $Q \equiv
Q_{\mn{acyc}}$. It thus suffices to show that $Q_{\mn{acyc}} \equiv
Q^{\mn{deco}}_{\mn{acyc}}$.

Using the definition of $Q^{\mn{deco}}_{\mn{acyc}}$, it can be shown
that $Q_{\mn{acyc}} \subseteq Q^{\mn{deco}}_{\mn{acyc}}$. To establish
the converse direction, assume towards a contradiction that there is a
$\Sigma_{\mn{full}}$-ABox \Amc such that $\Amc \models
Q_{\mn{acyc}}^{\mn{deco}}(a)$ but $\Amc \not \models
Q_{\mn{acyc}}(a)$. Then $\Amc \not\models Q'(a)$. Take a model \Imc of
\Amc and $\Tmc$ such that $\Imc \not\models C(a)$.  We have $\Imc
\models q_{\mn{acyc}}^{\mn{deco}}(a)$, thus $\Imc \models
p(x)|_{\mn{dreach}(p}(a)$ for some \Tmc-decoration $p(x)$ of a CQ in
$q_{\mn{acyc}}(x)$. Let $h$ be a homomorphism from
$p(x)|_{\mn{dreach}(p)}$ to \Imc with $h(x)=a$. 

To finish the proof, it suffices to show that we can construct from
$\Imc$ a model $\Imc'$ of \Tmc such that $\Imc' \not\models C(a)$ and
$\Imc' \models q_{\mn{acyc}}(a)$. In fact, we can then take a
homomorphism $h'$ from a CQ in $q_{\mn{acyc}}(x)$ to $\Imc'$ with
$h'(x)=a$ and let $\Amc'$ be $\Imc'$ restricted to the range of $h'$,
viewed as an ABox.  Clearly, $\Amc' \models Q_{\mn{acyc}}(a)$ since
already $\Amc' \models
(\emptyset,\Sigma_{\mn{full}},q_{\mn{acyc}}(x))(a)$. Moreover, $\Imc'$ is a model
of $\Amc'$ and thus $\Amc' \not\models Q'(a)$, in contradiction to
$Q'$ being equivalent to $Q_{\mn{acyc}}$.

It thus remains to construct $\Imc'$. Informally, we do this by adding
to \Imc the part of $p(x)$ that is not reachable from the answer
variable along a directed path. 
By the second condition of $\Tmc$-decorations, there is a
model \Jmc of \Tmc and a homomorphism~$h'$ from $p(x)$ to \Jmc.
We can assume that \Imc and \Jmc have disjoint domains.

Let $\Imc'$ be the disjoint union of $\Imc$ and $\Jmc$, extended as
follows: for every atom $r(y_1,y_2)$ in $p(x)$ with $y_1 \notin
\mn{dreach}(p)$ and $y_2 \in \mn{dreach}(p)$, add $(h'(y_1),h(y_2))$
to $r^{\Imc'}$. It can be verified that the map $h''$ defined by
setting $h''(y) = h(y)$ for all $y \in \mn{dreach}(p)$ and $h''(y) =
h'(y)$ for all variables $y$ in $p$ that are not in $\mn{dreach}(p)$
is a homomorphism from $p(x)$ to $\Imc'$ with $h''(x)=a$. Thus,
$\Imc' \models q_{\mn{acyc}}(a)$ as desired. It thus remains to show
that $\Imc'$ is a model of \Tmc and that  $\Imc' \not\models C(a)$.
This is a consequence of the following:
\begin{enumerate}

\item[(a)] for all \ALC-concepts $C$ and all $d \in \Delta^\Imc$,
  $d \in C^\Imc$ iff $d \in C^{\Imc'}$;

\item[(b)] for all subconcepts $C$ of a concept in \Tmc and all $d \in
  \Delta^\Jmc$, $d \in C^\Jmc$ iff $d \in C^{\Imc'}$.

\end{enumerate}
Both points are proved by induction on the structure of $C$. This is
straightforward for (a) since for every element $d \in \Delta^\Imc$,
the subinterpretation of $\Imc$ induced by the set of elements
reachable from $d$ in \Imc by traveling roles in the forwards
direction is identical to the corresponding subinterpretation of
$\Imc'$ (and since \ALC-concepts do not admit inverse roles). For~(b),
it is important to observe that if we have added $(h'(y_1),h(y_2))$ to
$r^{\Imc'}$ in the construction of $\Imc'$, then then $h'(y_1)$ has an
$r$-successor $d$ in $\Jmc$ such that for all subconcepts $C$ of a
concept in \Tmc, $d \in C^\Jmc$ iff $h(y_2) \in C^\Imc$. In fact, this
is a consequence of the decoration of every variable in $p(x)$ with
such concepts: when choosing $d=h'(y_2)$, the stated condition
must be satisfied.

%
%

\smallskip
 
``$5 \Rightarrow 4$''.  We have that $Q \equiv (\Tmc, \Sigma, q'(x))$, where each $q'_i$ is an $x$-acyclic, accessible eCQ. We can apply Lemma~\ref{lem:rewrAlgALC} to obtain an (\Lmc, IQ)-rewriting. 

\smallskip

For the case with the universal role, it is enough to
show that $3 \Rightarrow 6$. Again, we can apply Lemma
\ref{lem:rewrAlgALC}.
\end{proof}

\thmemptyTBox*
\noindent
\begin{proof}
For the statement at Point~1, the  ``if'' direction is a consequence of  Lemma~\ref{lem:rewrAlgALCI}, while for the statement at Point~2, the same direction is a consequence of Lemma~\ref{lem:rewrAlgALC} (and similarly for the cases where the universal role is present). We will thus show the  ``only if'' direction in each case. We first show that $\Lmc$-IQ rewritability of a UCQ $q(x)$, for every $\Lmc \in \{\ALCI, \ALC, \ALCI^u, \ALC^u \}$, implies the existence of a subquery $q'(x)$ of $q(x)$ that is $x$-acyclic and equivalent to $q(x)$. 

A \emph{homomorphism minimal CQ} (also \emph{hom-minimal}) is a CQ which does not admit any equivalent strict subquery. 

\begin{claimn}
Let $q$ and $q'$ be two CQs such that $q \equiv q'$ and $q'$ is hom-minimal. Then: 
\begin{enumerate}
\item $q'$ is a subquery of $q$;
\item $q'$ is a contraction of $q$. 
\end{enumerate}
\end{claimn}

\noindent
\emph{Proof of claim}. 
Consider any homomorphisms $h_1$ and $h_2$ from $q'$ to $q$ and from $q$ to $q'$, respectively. Then $h_1$ must be injective and $h_2$ must be surjective (otherwise $h_1\circ h_2$ is a non-injective homomorphism from $q'$ to itself, and thus $q'$ is not hom-minimal). 
The existence of $h_1$ implies that $q'$ is a subquery of $q$, while the existence of $h_2$ implies that $q'$ is a contraction of $q$.

\begin{claimn}
Let $p(x)$ be a CQ in $q_{\mn{acyc}}(x)$. Then, there exists a hom-minimal CQ $p'(x)$ in $q_{\mn{acyc}}(x)$ such that   $p(x) \equiv  p'(x)$.
\end{claimn}

\noindent
\emph{Proof of claim}. We show that, in fact, every hom-minimal subquery $p'(x)$ of $p(x)$ which is equivalent to $p(x)$ is a CQ in $q_{\mn{acyc}}(x)$. From Claim 1, $p'(x) \equiv p(x)$ and $p'(x)$ being hom-minimal, implies that $p'(x)$ is a contraction of $p(x)$. As $p(x)$ is a contraction of a CQ in $q(x)$, it follows that $p'(x)$ is itself a contraction of some CQ in $q(x)$. As $p(x)$ is $x$-acyclic and $p'(x)$ is a subquery of $p(x)$, it follows that $p'(x)$ is $x$-acyclic. Thus, $p'(x)$ is an $x$-acyclic contraction of some CQ in $q(x)$, or, in other words, $p'(x)$ is a CQ in $q_{\mn{acyc}}(x)$.

\medskip

Assume now $\Lmc$-IQ rewritability of $q(x)$. By inspecting Point~(a) in the proof of  direction ``1 $\Rightarrow$ 3'' of Theorem~\ref{thm:central}, we observe that $q(x) \equiv q_{\mn{acyc}}(x)$. 

For didactic purposes, we first consider the case where $q(x)$ is a CQ. Then, there must be a CQ $p(x)$ in the UCQ $q_{\mn{acyc}}(x)$ such that $q(x) \equiv p(x)$. From Claim 2, there must be some CQ $p'(x)$ in $q_{\mn{acyc}}(x)$ which is hom-minimal and equivalent to $p(x)$, and thus also to $q(x)$. From Claim 1, it follows that $p'(x)$ is a subquery of $q(x)$, and from the fact that $p'(x)$ is a CQ in $q_{\mn{acyc}}(x)$, it follows that $p'(x)$ is $x$-acyclic. 

We now consider the case where $q(x)$ is a UCQ. Let $q_1(x),\dots,q_k(x)$ be the CQs in $q(x)$
 that are \emph{minimal} in $q(x)$ in the following sense: for all CQs $p(x)$ in $q(x)$, $q_i(x) \subseteq p(x)$ implies $q_i(x) \equiv p(x)$. Take such a minimal CQ $q_i(x)$. Since $q(x) \equiv q_{\mn{acyc}}(x)$, there must be a CQ $p_i(x)$ in $q_{\mn{acyc}}(x)$ such that $q_i(x) \subseteq p_i(x)$. By construction of $q_{\mn{acyc}}(x)$, $p_i(x)$ must be the contraction of some CQ $\widehat q_i(x)$ in $q(x)$ and thus $p_i(x) \subseteq \widehat q_i(x)$.  We obtain $q_i(x) \subseteq \widehat q_i(x)$ and thus $\widehat q_i(x) \equiv q_i(x)$ and consequently $q_i(x) \equiv p_i(x)$. From Claim 2, there must be some hom-minimal query $p'_i(x) \in q_{\mn{acyc}}(x)$ such that $p'_i(x) \equiv p(x)$. Then, $q_i(x) \equiv p'_i(x)$ and from Claim 1,  $p'_i(x)$ is a sub-query of $q_i(x)$. Let $q_{\mn{acyc}}^-(x)$ be the restriction of $q_{\mn{acyc}}(x)$ to the chosen CQs $p'_1(x),\dots,p'_k(x)$. Clearly, $q_{\mn{acyc}}^-(x)$ is equivalent to $q(x)$.  
 
Now we concentrate on the ``only if'' direction for Point~1, i.e. the case where $\Lmc$ is $\ALCI$, and thus $q(x)$ is \ALCI-IQ rewritable. We already know that $q_{\mn{acyc}}^-(x)$ is equivalent to $q(x)$, thus $q_{\mn{acyc}}^-(x)$ is also \ALCI-IQ rewritable. From the proof of direction~``1 $\Rightarrow$ 3'' Point~(b) in Theorem~\ref{thm:central} , \ALCI-IQ rewritability implies  $q(x) \equiv q^{\mn{con}}(x)$, and thus also $q^{\mn{con}}(x) \equiv (q_{\mn{acyc}}^-)^{\mn{con}}(x)$ and $q(x) \equiv (q_{\mn{acyc}}^-)^{\mn{con}}(x)$.  It is easy to see that $(q_{\mn{acyc}}^-)^{\mn{con}}(x)$ is an $x$-acyclic connected subquery of $q(x)$. Point~2 and the cases with universal roles are treated similarly. 

\end{proof}


\lememptyasapprox*

\noindent
\begin{proof}\ 
  We show the result in the case where $\Lmc=\ALCI$. All other cases follow similarly. 
  
  Assume that $q(x)$ is \ALCI-IQ-rewritable and let
  $Q=(\Tmc,\Sigma,q(x))$ be an OMQ from $(\Lmc\Hmc,\text{UCQ})$. As a
  consequence of Theorem~\ref{thm:centralnoTBox}, there is a
  subquery $q'(x)$ of $q(x)$ that is $x$-acyclic, connected and
  equivalent to $q(x)$. Let  $Q'=(\Tmc,\Sigma,q'(x))$.  Then, $Q\equiv Q'$ and  according to Lemma~\ref{lem:rewrAlgALCI}, $Q'$ (and thus also $Q$) is rewritable into an OMQ from $(\ALCI, \text{IQ})$. 
  \end{proof}

\section{Proofs for Section 4}

For the proofs in this section, we recall that every CQ $q$ can be
viewed in a straightforward way as an ABox $\Amc_q$ by viewing
the atoms as assertions and the variables as individual names.

\thmemptyTComplexity*
\noindent
\begin{proof}\ Let $\Lmc \in \{\ALC, \ALCI, \ALC^u, \ALCI^u\}$. We
  start with the upper bound, that is, given a UCQ $q(x)$, it is in
  {\sc NP} to decide whether $q(x)$ is \Lmc-IQ-rewritable. 
  
We guess a subquery $q'(x)$ of the original query $q(x)$ and check whether $q'(x) \equiv q(x)$. This is the case when for every CQ $p(x)$ in $q(x)$ there exists a CQ $p'(x)$ in $q'(x)$  such that $p(x) \subseteq p'(x)$ and vice versa. 
We guess for every CQ $p(x)$ in $q(x)$ a target CQ $p'(x)$ in $q'(x)$ and a potential homomorphism $h_{p}: \Ind(\Amc_{p}) \to \Ind(\Amc_{p'})$. We also guess for every CQ $p'(x)$ in $q'(x)$ a target CQ $p(x)$ in $q(x)$ and a potential homomorphism $h_{p'}:
\Ind(\Amc_{p'}) \to \Ind(\Amc_{p})$. We then check that every $h_p$ and every $h_p'$ is an actual homomorphism. If this is the case, $q'(x) \equiv q(x)$ and, provided that $q'(x)$ fulfills the additional conditions in the $\Lmc$-IQ characterisation  from Theorem~\ref{thm:emptyTComplexity} ($x$-acyclicity, connectedness and/or $x$-accesibility), $q(x)$ is $\Lmc$-IQ rewritable. As the size of our guess is polynomial in the size of $q(x)$ and all checks can be performed in polynomial time, we obtain the desired upper bound. 

\medskip
  
To show {\sc NP}-hardness of whether a given CQ is \Lmc-IQ-rewritable,
we employ a reduction from the 3-colorability problem (3COL). Let
$G=(V,E)$ be an undirected graph, let $q_G$ be $G$ viewed as a
conjunctive query where every $\{v_1,v_2\} \in E$ is represented by
two atoms $r(v_1,v_2),r(v_2,v_1)$, and choose a $v \in V$. Let 
$$
\begin{array}{rcl}
   q(x_0)&=&\exists \ybf \, q_G \wedge  r(x_0,v) \wedge  r(v,x_0) \, \wedge \\[1mm]
&& \ \ \ \ \,  \bigwedge \{ r(x_i,x_j) \mid i,j \leq 2 \text{ with } i \neq j \} 
\end{array}
$$
  where \ybf contains all elements of $V$ (as variables) as well as
  the fresh variables $x_1$ and $x_2$.
  
%

\begin{claim}
  $q(x_0)$ is $\Lmc$-IQ-rewritable iff $G$ is 3-colorable.
\end{claim}

\noindent
\emph{Proof of claim}. For the ``if'' direction, assume that $G$ is
3-colorable. Then $G$ admits a homomorphism into the 3-clique (without
reflexive loops).  Consequently, $q_0(x)$ is homomorphically
equivalent to the restriction $q_{3C}(x_0)$ of $q(x_0)$ to the variables
$x_0,x_1,x_2$. In particular, $q(x_0)$ and $q_{3C}(x_0)$ are then
equivalent in the sense of query containment. Since $q_{3C}(x_0)$ is
$x_0$-acyclic and $x_0$-accessible, by Lemma~\ref{lem:rewrAlgALC} it
is rewritable into an \ALC-IQ. 

Conversely, assume that $G$ is not 3-colorable.  By
Lemma~\ref{lem:rewrAlgALC}, it suffices to show that any subquery
$p(x_0)$ of $q(x_0)$ that is equivalent to $q(x_0)$ is not
$x_0$-acyclic. Thus let $p(x_0)$ be such a subquery. There is no
homomorphism $h$ from $p(x_0)$ to $q_{3C}(x_0)$ since the equivalence
of $p(x_0)$ and $q(x_0)$ implies the existence of a homomorphism $h'$
from $q(x_0)$ to $p(x_0)$ and composing $h'$ with $h$ would establish
3-colorability of $G$. It is easy to verify, though, that when
$p(x_0)$ contains no cycle that does not pass through any of
$x_0,x_1,x_2$, then there is such a homomorphism $h$. Consequently,
$p(x_0)$ is not $x_0$-acyclic.
\end{proof}

\thmfullSigComplexity*
\noindent
\begin{proof}\ 
%
  We start with the lower bounds. Points~1 and~2 are treated
  uniformly. In fact, for $\Lmc \in \{\ALC,\ALCI\}$, we reduce a
  special case of OMQ evaluation in $(\Lmc,\text{CQ})$ to
  IQ-rewritability in $(\Lmc,\text{CQ})$ where \emph{OMQ evaluation}
  in $(\Lmc,\Qmc)$ means to decide, given an OMQ
  $Q=(\Tmc,\Sigma_{\mn{full}},q(\xbf))$ from $(\Lmc,\Qmc)$, an ABox
  \Amc, and a tuple \abf whether $\Amc\models Q(\abf)$. The mentioned
  special case is that OMQs are Boolean and \Amc takes the form $\{ A(a)
  \}$ and we refer to this as \emph{singleton BOMQ evaluation}.

Singleton BOMQ evaluation is 2\ExpTime-hard in \ALCI
\cite{DBLP:conf/cade/Lutz08}. We observe that it is \ExpTime-hard in
\ALC since concept (un)satisfiability w.r.t. $\ALC$-TBoxes is
\ExpTime-hard \cite{Schild1991} and an \ALC-concept $C$ is
unsatisfiable w.r.t.\ an \ALC-TBox $\Tmc$ iff $\{A(a)\} \models (\Tmc
\cup \{A \sqsubseteq C\},\Sigma_{\mn{full}},\exists y \, D(y))$ where
$A$ and $D$ are fresh concept names.

Now for the reduction to IQ-rewritability. Let
$Q=(\Tmc,\Sigma_{\mn{full}},q())$ be an OMQ from $(\Lmc,\text{CQ})$,
$\Lmc \in \{ \ALC, \ALCI \}$, and let $\Amc = \{A(a)\}$ be an ABox.
Further, let $q'(x)$ is the extension of $q()$ with the atom $A(x)$,
$x$ a fresh answer variable. It is important to note that $q'(x)$ is a
disconnected CQ.
\\[2mm]
{\bf Claim.} $\Amc \models Q$ iff $Q'=(\Tmc,\Sigma_{\mn{full}},q'(x))$
is IQ-rewritable.
\\[2mm]
\emph{Proof of claim.}  If $\Amc \models Q$, then $Q'$ is equivalent
to $(\Tmc,\Sigma_{\mn{full}},A(x))$ which is from
$(\Lmc,\text{IQ})$. Conversely, assume that $\Amc \not\models Q$.  The
query $Q^{\mn{con}}$ from Point~(b) in the proof of the ``$1
\Rightarrow 3$'' direction of Theorem~\ref{thm:central}, applied to
$Q'$, is exactly $(\Tmc,\Sigma_{\mn{full}},A(x))$. As shown there,
IQ-rewritability of $Q$ implies $Q \equiv Q^{\mn{con}}$, in
contradiction to $\Amc \not\models Q$. This finishes the proof 
of the claim.

\medskip

For the upper bounds, we use the characterizations from
Theorem~\ref{thm:central} and Theorem~\ref{thm:alc}: deciding
IQ-rewritability in $(\ALCHI,\text{UCQ})$ amounts to checking
containment between $Q$ and $(\Tmc, \Sigma_{\mn{full}},
q^{\mn{con}}_{\mn{acyc}}(x))$ while deciding IQ-rewritability in
$(\ALCH,\text{UCQ})$ amounts to checking containment between $Q$ and
respectively $(\Tmc, \Sigma_{\mn{full}},
q^{\mn{deco}}_{\mn{acyc}}(x))$. Note that the two involved OMQs share
the same TBox and are based on the full ABox signature. There is also
an initial emptiness check, which however is just another containment
check. We thus have to argue that these containment checks can be
carried out in 2\ExpTime and \NExpTime, respectively. 

We start with the case of $(\ALCHI,\text{UCQ})$ and first observe that
containment in $(\ALCHI,\text{UCQ})$ is in 2\ExpTime. In fact, it is
shown in \cite{DBLP:conf/cade/Lutz08} that OMQ evaluation in
$(\ALCHI,\text{CQ})$ is in 2\ExpTime and the algorithm given there is
straightforwardly extended to $(\ALCHI,\text{UCQ})$. It follows that
containment between an OMQ $Q_1=(\Tmc,\Sigma_{\mn{full}},q_1(x_1))$
from $(\ALCHI,\text{CQ})$ in an OMQ $Q_2=(\Tmc,\Sigma_{\mn{full}},q_2(x_2))$
from $(\ALCHI,\text{UCQ})$ is in 2\ExpTime since $Q_1 \subseteq Q_2$
iff $\Amc_{q_1} \models Q_2(x_1)$. 

We next observe how this can be lifted to containment in
$(\ALCHI,\text{UCQ})$. In fact, it suffices to show that for
$Q_i=(\Tmc,\Sigma_{\mn{full}},q_i)$ from $(\ALCHI,\text{UCQ})$, $i \in
\{1,2\}$, and $q_1 = p_1 \vee \cdots \vee p_k$,
we have $Q_1 \subseteq Q_2$
iff $(\Tmc,\Sigma_{\mn{full}},p_i) \subseteq Q_2$ for all $i \in
\{1,\dots,k\}$. 
The ``if'' direction is trivial. For the ``only if''
direction, we argue as follows. Assume that
$(\Tmc,\Sigma_{\mn{full}},p_i) \subseteq Q_2$ for all $i \in
\{1,\dots,k\}$.  Let \Amc be an ABox and $a \in \mn{ind}(\Amc)$ such
that $\Amc \models Q_1(a)$. It suffices to show that $\Imc \models
q_2(a)$ for every finite model \Imc of \Amc and $\Tmc$. Let \Imc be
such a model and let $\Amc_\Imc$ be \Imc viewed as an ABox. Since
$\Amc \models Q_1(a)$, we must have $\Imc \models p_j(a)$ for some $j
\in \{1,\dots,k\}$. Then clearly also $\Amc_\Imc \models
p_j(a)$. Since $(\Tmc,\Sigma,p_j) \subseteq Q_2$, this yields
$\Amc_\Imc \models Q_2(a)$. Since $\Imc$ is a model of $\Amc_\Imc$ and
\Tmc, from this we obtain $\Imc \models q_2(a)$, as required.

The argument is not yet complete since the UCQs
$q^{\mn{con}}_{\mn{acyc}}(x)$ can be exponentially large. In fact, it
may contain exponentially many CQs, but each CQ is only of polynomial
size. For checking $(\Tmc,
\Sigma_{\mn{full}},q^{\mn{con}}_{\mn{acyc}}(x)) \subseteq Q$, using
the above argument we can use exponentially many containment checks
between an OMQ from $(\ALCHI,\text{CQ})$ and an OMQ from
$(\ALCHI,\text{UCQ})$, both of polynomial size. The overall complexity
is thus 2\ExpTime, as required. For checking $Q \subseteq (\Tmc,
\Sigma_{\mn{full}},q^{\mn{con}}_{\mn{acyc}}(x))$, we observe that, by
Lemma~\ref{lem:rewrAlgALCI}, $(\Tmc, \Sigma_{\mn{full}},q^{\mn{con}}_{\mn{acyc}}(x))$
is rewritable into an equivalent OMQ $(\Tmc, \Sigma_{\mn{full}},C(x))$
with $C(x)$ an \ALCHI-IQ (independently of the properties of $Q$) and
such that the size of $C(x)$ is polynomial in the size of
$q^{\mn{con}}_{\mn{acyc}}(x)$, which in turn is single exponential in
the size of $Q$. We can thus replace the check $Q \subseteq (\Tmc,
\Sigma_{\mn{full}},q^{\mn{con}}_{\mn{acyc}}(x))$ with $Q \subseteq
(\Tmc, \Sigma_{\mn{full}},C(x))$. This boils down to deciding OMQ
entailment in $(\ALCHI,IQ)$, which is in \ExpTime. So despite
$C(x)$ being of (single) exponential size, we achieve 2\ExpTime overall complexity.


 \smallskip

 For the case of $(\ALCH,\text{UCQ})$, the argument is essentially the
 same. However, as also shown in \cite{DBLP:conf/cade/Lutz08} OMQ
 evaluation in $(\ALCH,\text{UCQ})$ is in \ExpTime and thus so is our
 basic containment check between an OMQ from $(\ALCH,\text{CQ})$ and an
 OMQ from $(\ALCHI,\text{UCQ})$. Therefore, the check $(\Tmc,
 \Sigma_{\mn{full}},q^{\mn{deco}}_{\mn{acyc}}(x)) \subseteq Q$ can be
 implemented in \ExpTime despite the exponential number of CQs in
 $q^{\mn{deco}}_{\mn{acyc}}$. It is not clear, however, how to
 implement the containment check $Q \subseteq (\Tmc,
 \Sigma_{\mn{full}},q^{\mn{deco}}_{\mn{acyc}}(x))$ in \ExpTime. We
 give a sketch of how it can be implemented in \coNExpTime.  In fact,
 what we have to implement in \coNExpTime is the evaluation of an OMQ
 $Q=(\Tmc,\Sigma_{\mn{full}},q(x))$ where $q(x)$ is a UCQ with
 exponentially many connected CQs, each of polynomial size. Let \Amc
 be an ABox and $a \in \mn{Ind}(\Amc)$.  By Lemma~\ref{lem:canmod},
 $\Amc \not\models Q(a)$ iff there is an extended forest model \Imc of
 \Amc and \Tmc such that $\Imc \not\models q(a)$. It is easy to see
 that we can further demand that (the tree-shaped parts of) \Imc be of
 outdegree polynomial in the size of \Tmc. Our \NExpTime algorithm for
 the complement of OMQ evaluation is as follows. Let $m$ be the
 maximum number of variables of a CQ in $q(x)$. We guess an initial
 piece of the extended forest model \Imc that consists of the `ABox
 part' of \Amc together with the tree-shaped parts restricted to depth
 $m+1$, along with a type adornment, that is, a function $\mu$ that
 assigns a $\Tmc$-type to every element of the guessed initial part in
 a way that is consistent with the initial part. Note that we guess an
 object of single exponential size here. Since the CQ in $q(x)$ are
 connected and thus are independent of the part of \Imc that lies
 beyond the guessed initial part, we can verify that \Imc can be
 extended to a full model by considering the type $\mu(d)$ for every
 leaf $d$ in the initial part on level $m+1$ and verifying that
 $\mu(d)$ is satisfiable with \Tmc. This can be implemented in
 \ExpTime.
\end{proof}

We introduce a preparatory lemma and notation for the proof of
Theorem~\ref{thm:alc}.  An \emph{atomic query (AQ)} is an IQ of the
form $A(x)$, with $A$ a concept name. A \emph{Boolean atomic query
  (BAQ)} is a query of the form $\exists x \, A(x)$, with $A$ a concept
name, and a \emph{Boolean conjunctive query (BCQ)} is a CQ of arity
zero.

\begin{lemma}
\label{lem:BIQtoAQ}
Let $Q=(\Tmc, \Sigma, \exists x \, C(x))$ be an OMQ from $(\ALCI,
\text{BAQ})$. Then, there exists an OMQ $Q'=(\Tmc', \Sigma, M(x))$
from $(\ALCI, \text{AQ})$,  such that for all $\Sigma$-ABoxes~$\Amc$,
$\Amc \models Q$ iff there is an $a \in \Ind(\Amc)$ such that $\Amc \models Q'(a)$. 
\end{lemma}
\begin{proof}
  We first observe that $\Amc \models Q$
  iff $\Amc' \models Q$, for some $\Amc' \in \mn{CON}_\Amc$. Thus, it
  is enough to show the statement of the lemma for connected
  $\Sigma$-ABoxes.

  Let $\Amc'$ be a connected $\Sigma$-ABox and $M$ a fresh concept
  name. Let $\Tmc'$ be the TBox obtained from $\Tmc$ by adding $C
  \sqsubseteq M$ and $\exists r. M \sqsubseteq M$ for every role $r$
  such that $r$ or the inverse of $r$ occurs in \Tmc. It can be
  verified that $Q'=(\Tmc', \Sigma, M(x))$ is as required.
\end{proof}

We also introduce a more fine-grained version of a complexity result
from \cite{DBLP:conf/kr/BourhisL16} which highlights that the
complexity of containment is double exponential only in the maximum
size of CQs in the input OMQs, but not in their number. This only
requires a careful analysis of the constructions in
\cite{DBLP:conf/kr/BourhisL16}.
\begin{theorem}
\label{thm:OMQcontainmentComplexRel}
Containment between OMQs from $(\ALCHI, \text{UCQ})$ is in
2\NExpTime. More precisely, for OMQs $Q_1=(\Tmc_1, \Sigma, q_1)$ and
$Q_2=(\Tmc_2, \Sigma, q_2)$ with arity $a$ and where $n_i$ is the
number of CQs in $q_1$ and $n_i$ the maximum size of a CQ in $q_i$, $i
\in \{1,2\}$, it can be decided in time $2^{2^{p(\mn{log}n_1
    + s_1 + \mn{log}n_2+ s_2 + \mn{log} |\Tmc|+\mn{log}
    \mn{log} a ) }}$,
$p$ a polynomial. 
\end{theorem}

\thmPSigComplexity*
\noindent
\begin{proof}\ We start with the lower bound for Point~1, using a
  reduction from OMQ emptiness in $(\ALC, \text{AQ})$ which is known
  to be $\NExpTime$-hard. Let $Q_0=(\Tmc, \Sigma, A(x))$ be an OMQ from
  this language. Also let 
  $$
q(x)=\exists y \, A(x) \wedge r(x, y)
  \wedge r(y,y),
  $$
  where $r$ is a role name that does not occur in
  $\Tmc$ and let $Q=(\Tmc, \Sigma \cup \{ r \}, q(x))$.  It suffices to show that
  $Q_0$ is empty iff $Q$ is IQ-rewritable.

  In fact, emptiness of $Q_0$ implies emptiness of $Q$ and thus
  $(\ALC, \text{AQ})$-rewritability. Conversely, assume that $Q_0$ is
  non-empty. To show that $Q$ is not IQ-rewritable, by
  Theorem~\ref{thm:alc} it suffices to show that $Q \not\equiv
  Q^{\mn{deco}}_{\mn{acyc}} := (\Tmc, \Sigma \cup \{r\},
  q_{\mn{acyc}}^{\mn{deco}}(x))$ where $q_{\mn{acyc}}^{\mn{deco}}(x)$
  is a UCQ in which every CQ contains the subquery $r(x,x)$.  Since
  $Q_0$ is non-empty, there is a $\Sigma$-ABox $\Amc$ and an $a \in
  \mn{ind}(\Amc)$ such that $\Amc \models Q_0(a)$.  Since $r$ does not
  occur in \Tmc, we can assume w.l.o.g.\ that it does not occur in
  \Amc as well.  Let $\Amc' = \Amc \cup \{ r(a,b), r(b,b) \}$. By
  definition of $Q(a)$, clearly $\Amc' \models Q(a)$. Moreover, $\Amc'
  \not\models Q^{\mn{deco}}_{\mn{acyc}}(a)$ because $\Amc'$ does not
  contain the assertion $r(a,a)$. \medskip

  To establish the lower bound for Point~2, we use a reduction from
  containment between an OMQ from $(\mathcal{ALCI},\text{BAQ})$ and an
  OMQ from $(\mathcal{ALCI},\text{BCQ})$. This problem has been shown
  to be 2{\sc NExpTime}-hard in \cite{DBLP:conf/kr/BourhisL16}.
  The reduction presented there uses different TBoxes in the two
  involved OMQs. However, by Theorem~3 in
  \cite{DBLP:conf/kr/BienvenuLW12}, we can assume w.l.o.g.\ that 
  they both share the same TBox. 
   
  Now for the reduction to IQ-rewritability.  Let
  $Q_1=(\Tmc,\Sigma,\exists x \,A(x))$ be an OMQ from
  $(\mathcal{ALCI},\text{BAQ})$ and $Q_2=(\Tmc,\Sigma,q())$ an OMQ
  from $(\mathcal{ALCI},\text{BCQ})$. We first show that $Q_1$ and
  $Q_2$, which are Boolean, can be replaced with unary OMQs. By Lemma~\ref{lem:BIQtoAQ}, we find an OMQ $Q'_1=(\Tmc',\Sigma,M(x))$ from
  $(\mathcal{ALCI},\text{AQ})$ such that for all $\Sigma$-ABoxes \Amc,
  $\Amc \models Q_1$ iff $\Amc \models Q'_1(a)$ for some $a \in
  \mn{ind}(\Amc)$.  Clearly, the construction from the proof of
  Lemma~\ref{lem:BIQtoAQ} is such that $Q_1$ is equivalent to
  $(\Tmc',\Sigma,\exists x \,A(x))$ and by choosing the fresh concept
  $M$ to also not occur in $q()$ we can further ensure that $Q_2$ is
  equivalent to $(\Tmc',\Sigma,q())$. Thus, we can assume that $Q_1$,
  $Q_2$, and $Q'_1$ all use the same TBox \Tmc and we can further
  assume $\Tmc$ contains a CI $\top \sqsubseteq N$ where $N$ is a
  concept name not occurring anywhere else, including $\Sigma$. Set
  $Q'_2 = (\Tmc,\Sigma,q'(x))$ where $q'(x)$ is $q()$ extended with
  the atom $N(x)$, $x$ a fresh (answer) variable. It can be verified
  that $Q_1 \subseteq Q_2$ iff $Q'_1 \subseteq Q'_2$.

  Now let $q_0(x)$ be $q()$ extended with the atom $M(x)$, $x$ a fresh
  (answer) variable. It is important to note that $q_0(x)$ is a
  disconnected CQ.
  \\[2mm]
  {\bf Claim.} $Q'_1 \subseteq Q'_2$ iff $Q=(\Tmc,\Sigma,q_0(x))$ is
  IQ-rewritable.
  \\[2mm]
  \emph{Proof of claim.}  If $Q'_1 \subseteq Q'_2$, then $Q$ is
  equivalent to $Q'_1$, which is from $(\ALCI,\text{IQ})$. Conversely,
  assume that $Q'_1 \not \subseteq Q'_2$. The query $Q^{\mn{con}}$
  from Point~(b) in the proof of the ``$1 \Rightarrow 3$'' direction
  of Theorem~\ref{thm:central}, applied to $Q$, is exactly $Q'_1$. As
  shown there, IQ-rewritability of $Q$ implies $Q \equiv
  Q^{\mn{con}}$, in contradiction to  $Q'_1 \not \subseteq Q'_2$.

 \medskip

 The upper bounds are a consequence of the characterization of
 IQ-rewritability in $(\ALCHI, \text{UCQ})$ from Theorem~\ref{thm:alc}
 in terms of query containment. Contaiment in $(\ALCHI, \text{UCQ})$
 is in 2\NExpTime \cite{DBLP:conf/kr/BourhisL16}. Note that our
 characterizations use UCQs with exponentially many CQs, each of which
 is of polynomial size so we cannot apply the contaiment complexity result as a black box. However, by perusing the refined OMQ containment complexity result from Theorem \ref{thm:OMQcontainmentComplexRel}, we observe that containment checking for OMQs is double exponential only in the size of the CQs in the UCQs while it is only exponential in the number of CQs.

Note that we also need an emptiness check beforehand: if the check
succeeds and the OMQ is empty, it is also rewritable and so we answer
`yes', if not we proceed to perform the containment check.  Emptiness
is simply a special case of containment, so we end up in the right complexity class. 
\end{proof}

\section{Proofs for Section 5}
%

We postpone the proof of Theorem~\ref{lthmUndecALCF} as we need the ultrafilter technique introduced in the proof of Theorem~\ref{thm:f}.
We start by discussing the rewritings given in Example~\ref{ex:fr} in more detail and present an additional
example. Recall that $p(x) = \exists y(s(x,y) \wedge  r(y,y))$ and that we consider 
the OMQ $Q=(\Tmc_{r},\Sigma_{\sf full},p(x))$ with $\Tmc_{r}=\{{\sf func}(r)\}$. We claim that
$Q_{r}=(\Tmc_{r},\Sigma_{\sf full},q_{r}(x))$ with
$$
 q_{r}(x) = (\forall s.\bigsqcup_{1\leq i \leq 3}P_{i}) \rightarrow (\exists s.(\bigsqcap_{1\leq i \leq 3}(P_{i}\rightarrow \exists r.P_{i}))
 $$
is a rewriting of $Q$. To prove this claim one can use the following straightforward three colorability argument: 
for every set $X$ of individual names in an ABox $\Amc$ which does not contain
an atom of the form $r(c,c)$ with $c\in X$ and in which $r$ is functional
(in the sense that for any $a$ there is at most one $b$ with $r(a,b)\in \Amc$) one can color the individual
names in $X$ with three different colors $P_{1},P_{2},P_{3}$ without having distinct $c_{1},c_{2}$ with $r(c_{1},c_{2})\in \Amc$ 
such that $P_{i}(c_{1})$ and $P_{i}(c_{2})$ for some $1\leq i\leq 3$.
We give an additional example illustrating this technique which is fundamental for our approach to
rewritability for $\mathcal{ALCIF}$ TBoxes.
\begin{example}
Let $\Tmc= \{{\sf func}(s_{1}), {\sf func}(s_{2})\}$ and 
$$
p(x) = \exists y,z (r(x,y) \wedge s_{1}(y,z) \wedge s_{2}(y,z))
$$
and obtain $p'(x)$ from $p(x)$ by adding the atom $s_{1}(z,y)$. Then
$(\Tmc,\Sigma_{\sf full},p(x))$ is not rewritable into an OMQ 
in $(\Tmc,\Sigma_{\sf full},q(x))$ with $q(x)$ a CI but $(\Tmc,\emptyset,\Sigma_{\sf full},p'(x))$
is rewritable into the OMQ $(\Tmc,\Sigma_{\sf full},q(x))$, where

\begin{eqnarray*}
q(x) &  = & \forall r.((\bigsqcup_{1\leq i \leq 3}P_{i}) \sqcap (\forall r.\forall s_{1}.\bigsqcup_{1\leq i \leq 3}Q_{i})\\
  & \rightarrow & (\exists r.(\bigsqcap_{1\leq i \leq 3} (P_{i} \rightarrow \exists s_{1}.(C \sqcap \exists s_{1}.P_{i}))
\end{eqnarray*}
where
$$
C=\bigsqcap_{1\leq i \leq 3}(Q_{i} \rightarrow \exists s_{1}.\exists s_{2}.Q_{i})
$$
\end{example}
We split the proof of Theorem~\ref{thm:f} into two parts and state the claims in such a way
that they cover the extensions discussed in the main text. Call a role $r$ \emph{functional w.r.t.~$\Tmc$} if 
${\sf func}(r)\in \Tmc$. The \emph{functional closure $\text{FC}_{p}(x_{0})$ of a variable $x_{0}$ in a CQ $p(x)$} 
is the set of all variables $x_{n}$ in $p(x)$ such that there is a functional path from $x_{0}$ to $x_{n}$ in $p(x)$. 
Let $\text{var}(p(x))$ denote the set of all variables in $p(x)$ and set 
$$
\text{nFC}_{p}(x)= \text{var}(p(x))\setminus \text{FC}_{p}(x)
$$
\begin{lemma}
Let $Q=(\Tmc,\Sigma_{\text{full}},q(x))$ be an OMQ from $(\mathcal{ALCIF},\text{UCQ})$. Then
$Q$ is rewritable into an OMQ $(\Tmc,\Sigma_{\text{full}},q'(x))$ with $q'(x)$ an \ALCI-IQ 
if there is a subquery $q'(x)$ of $q(x)$ that is f-acyclic, connected, and equivalent to $q(x)$.

When $\mathcal{ALCIF}$ is replaced with $\mathcal{ALCIF}^{u}$, then the implication 
holds without the connectedness assumption.
\end{lemma}
\begin{proof}
Fix an f-acyclic and connected UCQ $q(x)$ and let $\Tmc$ be a $\mathcal{ALCIF}$ TBox.
Let $Q=(\Tmc,\Sigma_{\sf full},q(x))$. We construct an $\mathcal{ALCI}$-IQ $C(x)$  such that
$Q \equiv Q'$ for $Q'= (\Tmc,\Sigma_{\sf full},C(x))$. Let $p(x)$ be a CQ in $q(x)$. 
A \emph{cluster} in $p(x)$ is any maximal subset $X$ of $\text{nFC}_{p}(x)$ such that there is a 
functional path from any $y\in X$ to any $y'\in X$.
Denote the set of clusters in $p(x)$ by $\mathfrak{C}$.
For any two clusters $X_{1},X_{2}$ we set $(X_{1},X_{2})\in E$ if there exist $y_{1}\in X_{1}$ and $y_{2}\in X_{2}$ 
such that there is an atom $r(y_{1},y_{2})$ in $p(x)$. We obtain an undirected graph $(\mathfrak{C},E)$ without self-loops. 
As $p(x)$ is f-acyclic, we obtain that
\begin{itemize}
\item $(\mathfrak{C},E)$ is acyclic;
\item for any $(X_{1},X_{2})\in E$ there is exactly one atom $r(y_{1},y_{2})$ in $p(x)$ with $y_{1}\in X_{1}$ and $y_{2}\in X_{2}$;
\item there does not exist an atom $r(y,y')$ in $p(x)$ such that neither $r$ nor $r^{-}$ are functional w.r.t.~$\Tmc$ and $y,y'\in X$ 
for a single cluster $X$ in $p(x)$.
\end{itemize}
It follows that we obtain from $p(x)$ a new CQ $p'(x)$ by repeatedly choosing and removing 
\begin{itemize}
\item atoms $r(x',y')$ with $x'\in \text{FC}_{p}(x)$ and $y'\in \text{nFC}_{p}(x)$ and
\item atoms $r(y',y'')$ with $y',y''$ contained in the same cluster $X$ and
\item atoms $r(x',x'')$ with $x',x''\in \text{FC}_{p}(x)$ 
\end{itemize}
such that $p'(x)$ is a tree-shaped CQ with answer variable $x$ still containing all variables in $p(x)$ and
\begin{itemize}
\item every $x'$ in $\text{FC}_{p}(x)$ is still reachable in $p'(x)|_{\text{FC}_{q}(x)}$ from $x$ along a functional path;
\item for every cluster $X$ there exists a $y_{X}\in X$ such that every $y\in X$ can be reached in $p'(x)|_{X}$
from $y_{X}$ along a functional path.
\end{itemize}
Now obtain a CQ $p_{\sf plain}(x)$ from $p'(x)$ by adding for every $y$ in $p'(x)$ a fresh atom $A_{y}(y)$ to $p'(x)$.
Obtain an eCQ $p_{\sf deco}(x)$ from $p_{\sf plain}(x)$ by adding
\begin{itemize}
\item for every atom $r(x_{1},x_{2})\in p(x)\setminus p'(x)$ with $x_{1},x_{2}\in \text{FC}_{p}(x)$ the compound `atom' $\exists r.A_{x_{2}}(x_{1})$ and 
\item for every atom $r(y_{1},x_{1})\in p(x)\setminus p'(x)$ with $y_{1}\in \text{nFC}_{p}(x)$ and $x_{1}\in \text{FC}_{p}(x)$ the compound atom
$\exists r.A_{x_{1}}(y_{1})$.
\end{itemize}
Replace in the $\mathcal{ELI}$ concept corresponding to $p_{\sf plain}(x)$ all occurrences of the symbol `$\exists$' by the symbol `$\forall$'
and denote the resulting $\mathcal{ALCI}$ concept by $C_{\sf plain}$. Take the $\mathcal{ELI}$ concept $D_{\sf deco}$ corresponding to $p_{\sf deco}(x)$.
Consider the $\mathcal{ALCI}$ concept $C_{\sf plain} \rightarrow D_{\sf deco}$. It should be clear that 
$(\Tmc,\Sigma_{\sf full},C_{\sf plain} \rightarrow D_{\sf deco}(x))$ is a rewriting of $(\Tmc,\Sigma_{\sf full},p(x))$ if all clusters in
$p(x)$ are degenerate in the sense that they consist of a single variable $y$ such that there is no atom of the form $r(y,y)$ in $p(x)$.
More generally, $(\Tmc,\Sigma_{\sf full},C_{\sf plain} \rightarrow D_{\sf deco}(x))$ is a rewriting of 
$(\Tmc,\Sigma_{\sf full},p^{\ast}(x))$ for the CQ $p^{\ast}(x)$ obtained from $p(x)$ by
removing all atoms $r(y_{1},y_{2})$ with variables $y_{1},y_{2}$ from a single cluster which are not in $p'(x)$. To obtain the rewriting 
we are after we thus still have to take care of those atoms.

Consider a non-degenerate cluster $X$. We find an ordering $r_{1}(x_{1}^{1},x_{2}^{1}),\ldots r_{n}(x_{1}^{n},x_{2}^{n})$ of the atoms in 
$p(x)|_{X}\setminus p'|_{X}$ such that all $r_{i}$ are functional (recall that no atoms $r(y,y')$ with neither $r$ nor $r^{-}$ functional
and $y,y'\in X$ exist) and for inductively defined sets of atoms  
\begin{eqnarray*}
p^{0}|_{X} & = & p'|_{X}\\
p^{i+1}|_{X} & = & p^{i}|_{X}\cup\{r_{i+1}(x_{1}^{i+1},x_{2}^{i+1})\}
\end{eqnarray*}
the following holds: for all $0 \leq i < n$, there is a functional path $y_{0}^{i},\ldots,y_{k}^{i}$ in $p^{i}|_{X}$ such that 
$y_{0}^{i}=y_{k}^{i}$, $y_{k-1}^{i}=x_{1}^{i+1}$ and $y_{k}^{i}=x_{2}^{i+1}$.  
Now take for every $0\leq i < n$ fresh concept names $P_{X,i}^{1}, P_{X,i}^{2},P_{X,i}^{3}$ and obtain $p_{\sf plain}^X(x)$ 
from $p_{\sf plain}(x)$ by adding the compound atoms 
$$
P_{X,i}^{1}\sqcup P_{X,i}^{2}\sqcup P_{X,i}^{3}(y_{0}^{i})
$$
Take functional roles $s_{0}^{i},\ldots,s_{k-2}^{i}$ and consider the query 
$$
s_{0}^{i}(y_{0}^{i},y_{1}^{i}),\ldots,s_{k-2}^{i}(y_{k-2}^{i},y_{k-1}^{i})\in p^{i}|_{X}
$$
and consider the CQ $p^{i}_{X}(y_{0}^{i})$ defined by taking the conjunction of the atoms
$s_{0}^{i}(y_{0}^{i},y_{1}^{i}),\ldots,s_{k-2}^{i}(y_{k-2}^{i},y_{k-1}^{i})$ and taking $y_{0}^{i}$ as the answer variable.
Define for $1\leq j \leq 3$ an eCQ $p^{i,j}_{X}(y_{0}^{i})$ by adding to $p^{i}_{X}(y_{0}^{i})$ the compound atom
$$
\exists r_{i+1}.P_{X,i}^{j}(y_{k-1}^{i})).
$$
Take the $\mathcal{ELI}$ concept $D_{X,i}^{j}$ corresponding to $p^{i,j}_{X}(y_{0}^{i})$, for $j=1,2,3$. 
Obtain the CQ $p_{\sf deco}^{X}(x)$ from $p_{\sf deco}(x)$ by adding the compound atoms
$$
\bigsqcap_{1\leq j \leq 3}(P_{X,i}^{j} \rightarrow D_{X,i}^{j})(y_{0}^{i})
$$
We do this for all non-degenerate clusters $X$ and obtain eCQs $p_{\sf plain}^{\ast}(x)$ and $p_{\sf deco}^{\ast}(x)$
by taking the conjunction of all $p_{\sf plain}^{X}(x)$ and $p_{\sf deco}^{X}(x)$, respectively. Now define the concepts 
$C_{\sf plain}^{\ast}$ and $D_{\sf deco}^{\ast}$ in the obvious way. It is not difficult to prove that
$(\Tmc,\Sigma_{\sf full},C_{\sf plain}^{\ast} \rightarrow D_{\sf deco}^{\ast}(x))$ is a rewriting of 
$(\Tmc,\Sigma_{\sf full},p(x))$. By constructing $C_{\sf plain}^{\ast}$ and $D_{\sf deco}^{\ast}$ for all CQs $p(x)$
in $q(x)$ and taking the disjunction of all $C_{\sf plain}^{\ast} \rightarrow D_{\sf deco}^{\ast}(x)$ we obtain
a rewriting of $q(x)$. The extension to $\mathcal{ALCIF}^{u}$ without the connectedness assumption is straightforward.
\end{proof}
For the proof of the other direction of Theorem~\ref{thm:f}, we require some preparation.
We use standard notation and results for ultrafilter extensions of interpretations~\cite{Blackb}.
For $U\subseteq \Delta^{\Imc}$ for an interpretation $\Imc$ we set $\overline{U}=\Delta^{\Imc}\setminus U$.
\begin{definition}
Let $\Imc$ be an interpretation. A set $\mathfrak{U} \subseteq 2^{\Delta^{\Imc}}$ is an \emph{ultrafilter} over $\Delta^{\Imc}$ if the
following conditions hold for all $U,V\subseteq \Delta^{\Imc}$:
\begin{itemize}
\item if $U,V\in \mathfrak{U}$, then $U\cap V\in \mathfrak{U}$;
\item if $U\in \frak{U}$ and $U\subseteq V$, then $V\in \mathfrak{U}$;
\item $U\in \mathfrak{U}$ iff $\overline{U}\not\in \mathfrak{U}$.
\end{itemize}
\end{definition}
For every $d\in \Delta^{\Imc}$, the set
$$
\mathfrak{U}_{d}= \{ X\subseteq \Delta^{\Imc} \mid d\in X\}
$$
is an ultrafilter, called the \emph{principal ultrafilter} generated by $d$ in $\Delta^{\Imc}$. An ultrafilter
$\mathfrak{U}$ for which there exists no $d\in \Delta^{\Imc}$ with $\mathfrak{U}=\mathfrak{U}_{d}$ is called
a \emph{non-principal} ultrafilter. It is known that for every set $\mathfrak{V}\subseteq 2^{\Delta^{\Imc}}$ with
the \emph{finite intersection property} (if $U_{1},\ldots,U_{n}\in \mathfrak{V}$, 
then $U_{1}\cap \cdots \cap U_{n}\not=\emptyset$) there exists an ultrafilter $\mathfrak{U}\supseteq \mathfrak{V}$. For a set $U\subseteq \Delta^{\Imc}$ and role $r$ we set
$$
(\exists r.U)^{\Imc}= \{ d\in \Delta^{\Imc}\mid \text{there exists $d'\in U$ with $(d,d')\in r^{\Imc}$}\}
$$
We are in the position now to define ultrafilter extensions of interpretations.
\begin{definition}
Let $\Imc$ be an interpretation. The \emph{ultrafilter extension} $\Imc^{\text{ue}}$ of $\Imc$ is defined
as follows:
\begin{itemize}
\item $\Delta^{\Imc^{\text{ue}}}$ is the set of ultrafilters over $\Delta^{\Imc}$;
\item $\mathfrak{U}\in A^{\Imc^{\text{ue}}}$ iff $A^{\Imc}\in \mathfrak{U}$, for all concept names $A$;
\item $(\Umf_{1},\Umf_{2})\in r^{\Imc^{\text{ue}}}$ iff $(\exists r.U)^{\Imc}\in \Umf_{1}$ for
all $U\in \Umf_{2}$, for all role names $r$.
\end{itemize}
\end{definition}
Observe the following equivalence for all principal ultrafilters $\Umf_{d},\Umf_{e}$ and all roles $r$:
$$
(d,e)\in r^{\Imc} \quad \Leftrightarrow \quad (\Umf_{d},\Umf_{e})\in r^{\Imc^{\text{ue}}}
$$
The fundamental property of ultrafilter extensions is the following
\emph{anti-preservation result}:
\begin{lemma}\label{lem:anti-preservation}
For all interpretations $\Imc$, $\mathcal{ALCI}$ concept $C$, and roles $r$:
\begin{itemize}
\item if $\Imc\models C(a)$, then $\Imc^{\text{ue}}\models C(U_{a})$.
\item if $r^{\Imc}$ is functional, then $r^{\Imc^{\text{ue}}}$ is functional. 
\end{itemize}
\end{lemma}
We are now in the position to prove the second part of Theorem~\ref{thm:f}.
\begin{lemma}
If an OMQ $Q=(\Tmc,\Sigma_{\text{full}},q(x))$ from $(\Fmc,\text{UCQ})$
is rewritable into an OMQ from $(\Fmc,\ALCI\text{-IQ})$, then there is a
subquery $q'(x)$ of $q(x)$ that is f-acyclic, connected, and equivalent to $q(x)$.

When $\mathcal{ALCI}$-IQ is replaced with $\mathcal{ALCI}^{u}$-IQ,
then the same equivalence holds except that connectedness is dropped.
\end{lemma}
\begin{proof}
Assume that $\Tmc$ contains functionality assertions only and $q(x)$ is a UCQ 
such that there does not exist an equivalent subquery $q'(x)$ of $q(x)$ which is 
f-acyclic and connected. We may assume that
\begin{itemize}
\item there is no homomorphism from any disjunct $p(x)$ of $q(x)$ to another 
disjunct $p'(x)$ of $q(x)$;
\item every homomorphism from any disjunct $p(x)$ into itself is surjective.
\end{itemize}
Take a disjunct $p(x)$ of $q(x)$ which is not f-acyclic or not connected. We consider the case that $p(x)$ is not 
f-acyclic but connected. The case that $p(x)$ is not connected is straightforward. We have a cycle
$r_{0}(x_{0},x_{1}),\ldots,r_{n-1}(x_{n-1},x_{n})$ in $p(x)$
such that $\text{FC}_{q}(x) \cap \{x_{0},\ldots,x_{n-1}\}=\emptyset$ and  
\begin{enumerate}
\item $r_{i}$ or $r_{i}^{-}$ is not functional w.r.t.~$\Tmc$ for some $i<n$ or
\item there exists no functional path $y_{0},\ldots,y_{m}$ in $p(x)$ with $x_{0}=y_{0}=y_{m}$ such that
$\{x_{0},\ldots,x_{n-1}\}\subseteq \{y_{0},\ldots,y_{m}\}$.
\end{enumerate}
The basic idea of the proof for both Point~1 and Point~2 is as follows. Assume there exists a 
a rewriting $(\Tmc,\Sigma_{\sf full},C(x))$ of $(\Tmc,\Sigma_{\sf full},q(x))$. 
\begin{itemize}
\item[(a)] Using the CQ $p(x)$ we construct an infinite ABox $\Amc$ such that $\Tmc,\Amc\not\models q(a)$; 
\item[(b)] By compactness of FO there exists a forest model $\Imc$ of $\Tmc$ and $\Amc$ with $\Imc\models \neg C(a)$; 
\item[(c)] Then $\Imc^{\text{ue}}\models \neg C(\Umf_{a})$ by Lemma~\ref{lem:anti-preservation};
\item[(d)] Moreover, $\Imc^{\text{ue}}$ is a model of a finite ABox $\Amc'$ with individual name $\Umf_{a}$ 
such that $\Tmc,\Amc'\models q(\Umf_{a})$. But this contradicts $\Imc^{\text{ue}}\models \neg C(\Umf_{a})$.
\end{itemize}
For Point~1, the construction of $\Amc$ is a variant of a construction 
given in~\cite{DBLP:journals/japll/KikotZ13}.
Consider the cycle $r_{0}(x_{0},x_{1}),\ldots,r_{n-1}(x_{n-1},x_{n})$ in $p(x)$ such that
for some $i$ neither $r$ nor $r^{-}$ are functional w.r.t.~$\Tmc$. We may assume that $i=0$.
Let $V$ be the connected component of $\{x_{0},\ldots,x_{n-1}\}$ in $\text{nFC}_{p}(x)$. Regard
the variables of $p(x)$ as individual names. Define an ABox $\Amc$ with individuals
$$
\text{FC}_{p}(x) \cup (\text{nFC}_{p}(x)\setminus V) \cup (V\times \mathbb{N})
$$
by setting:
\begin{itemize}
\item for all concept names $A$ and variables $y\in \text{FC}_{p}(x) \cup (\text{nFC}_{p}(x)\setminus V)$: $A(y)\in \Amc$ iff $A(y)$ is in $p(x)$;
\item for all concept names $A$, variables $y\in V$, and $i\in \mathbb{N}$: $A(y,i)\in \Amc$ iff $A(y) \in q(x)$;
\item for all roles $r$ and variables $y,z\in \text{FC}_{p}(x) \cup (\text{nFC}_{p}(x)\setminus V)$: $r(y,z)\in \Amc$ iff $r(y,z)$ is in $p(x)$;
\item for all roles $r$, variables $y\in \text{FC}_{p}(x) \cup (\text{nFC}_{p}(x)\setminus V)$ and $z\in V$, and 
$i\in \mathbb{N}$: $r(y,(z,i))\in \Amc$ iff $r(y,z)$ is in $p(x)$;
\item for all roles $r$, variables $y,y'\in V$, and $i\in \mathbb{N}$: $r((y,i),(y',i))\in \Amc$ iff $r(y,y')$ is in $p(x)$ and $r_{0}(x_{0},x_{1})\not=r(y,y')$;
\item for all roles $r$, variables $y,y'\in V$, and $i\not=j\in \mathbb{N}$, $r((y,i),(y',j))\in \Amc$ iff $i<j$ and $r_{0}(x_{0},x_{1})=r(y,y')$.
\end{itemize}
We now check that $\Amc$ satisfies Points~(a) to~(d).

\medskip

Point~(a). Clearly all $r$ functional w.r.t.~$\Tmc$ are functional in $\Amc$. Thus, $\Tmc,\Amc\models q(x)$
iff $\emptyset,\Amc\models p'(x)$ for some disjunct $p'(x)$ of $q(x)$. It therefore suffices to show that there
is no homomorphism from any disjunct $p'(x)$ of $q(x)$ to $\Amc$ mapping $x$ to $x$. It has been observed
in~\cite{DBLP:journals/japll/KikotZ13} already that the mapping $\pi$ from $\Amc$ to $p(x)$ mapping every variable 
to itself and every $(y,i)\in V \times \mathbb{N}$ to $y$ is a homomorphism. Using this observation it has been shown 
that there is no homomorphism from $p(x)$ to $\Amc$ as the composition of $\pi$ with such a homomorphism would be a non-surjective
homomorphism mapping $x$ to $x$ which contradicts our assumptions. It also follows that there is no homomorphism from
any disjunct $p'(x)$ of $q(x)$ distinct from $p(x)$ to $\Amc$ as the composition of $\pi$ with such a homomorphism would be a
homomorphism from $p'(x)$ to $p(x)$ mapping $x$ to $x$.

\medskip

Point~(b). Assume there exists no model of $\Tmc$ and $\Amc$ with $\Imc\models \neg C(x)$. Then, by compactness,
there exists a finite subset $\Amc'$ of $\Amc$ such that there exists no model of $\Tmc$ and $\Amc'$ with 
$\Imc\models \neg C(x)$. But then $\Tmc,\Amc'\models q(x)$ which contradicts Point~(a) and the assumption that 
$(\Tmc,\Sigma_{\sf full},C(x))$ is a rewriting of $(\Tmc,\Sigma_{\sf full},q(x))$.
  
\medskip

Point~(c). This is by Lemma~\ref{lem:anti-preservation}.

\medskip

Point~(d). It follows directly from~\cite{DBLP:journals/japll/KikotZ13} that $\Imc^{\text{ue}}$ contains a 
homomorphic image of $p(x)$ under a homomorphism mapping
$x$ to $\Umf_{x}$. Regard this image as an ABox $\Amc'$. Then $\Tmc,\Amc'\models q(\Umf_{a})$.
This finishes the proof if Point~1 holds.
 
\bigskip

Now suppose that Point~2 does not hold. Thus, there exists no functional path $y_{0},\ldots,y_{m}$ in $p(x)$ with 
$x_{0}=y_{0}=y_{m}$ such that $\{x_{0},\ldots,x_{n-1}\}\subseteq \{y_{0},\ldots,y_{m}\}$.
Let $p^{n}=p|_{\text{nFC}_{p}(x)}$. We observe the following

\medskip
\noindent
\emph{Claim 1.} There exist $V_{1}\subseteq \text{nFC}_{p}(x)$ of the form $V_{1}= \text{FC}_{p^{n}}(x')$  
for some $x'$ in $p(x)$ such that for $V_{2}:= \text{nFC}_{q}(x)\setminus V_{1}$ there are  
$y_{1},y_{2}\in V_{2}$ such that there is a path from $y_{1}$ to $y_{2}$ in $V_{2}$ and $z_{1},z_{2}\in V_{1}$ 
such that there are distinct $s_{1}(y_{1},z_{1}),s_{2}(y_{2},z_{2})\in q(x)$ with $s_{1},s_{2}$ functional w.r.t.$\Tmc$.

\medskip
For the proof of Claim~1 take the cycle $r_{0}(x_{0},x_{1}),\ldots,r_{n-1}(x_{n-1},x_{n})$ in $p(x)$.
There must exist $x_{i}$ such that some $x_{j}$ with $x_{i}\not= x_{j}$ is not in $\text{FC}_{p}(x_{i})$.
Let $V_{1}=\text{FC}_{p^{n}}(x_{i})$. Then we find $x_{j}\in \text{FC}_{p^{n}}(x_{i})$ such that 
$x_{j+1}\not\in \text{FC}_{p^{n}}(x_{i})$
and we find a path (possibly of length $0$) from $x_{j+1}$ to some $x_{j'}$ within $V_{2}$ such that there
exist $r'$ and $x_{j''}\in V_{1}$ with $r'(x_{j'},x_{j''})\in p(x)$. Then $s_{1}(y_{1},z_{1}):= r_{j}^{-}(x_{j+1},x_{j})$
and $s_{2}(y_{2},z_{2}):=r'(x_{j'},x_{j''})$ are as required.

\medskip

We define an ABox $\Amc$ with individual names 
$$
\text{FC}_{p}(x) \cup (V_{1}\times \mathbb{N}) \cup (V_{2} \times I)
$$
where
$$
I = \{ (\beta,E) \mid E\subseteq \mathbb{N}, |E|=|\mathcal{S}|, \beta: E \rightarrow \mathcal{S} \text{ bijective }\}
$$
and
$$
\mathcal{S}= \{ r(y,z)\in p(x) \mid z\in V_{1}, y\in V_{2}\}
$$
as follows:
\begin{itemize}
\item for all concept names $A$ and variables $y\in \text{FC}_{p}(x)$: $A(y)\in \Amc$ iff $A(y)$ is in $p(x)$;
\item for all roles $r$ and variables $y,z\in \text{FC}_{p}(x)$: $r(y,z)\in \Amc$ iff $r(y,z)$ is in $p(x)$;
\item for all concept names $A$, variables $y\in V_{1}$, and $i\in \mathbb{N}$: $A(y,i)\in \Amc$ iff $A(y)$ is in $p(x)$;
\item for all concept names $A$, variables $y\in V_{2}$, and $i\in I$: $A(y,i)\in \Amc$ iff $A(y)$ is in $p(x)$;
\item for all roles $r$, variables $y,z\in V_{1}$ and $i\in \mathbb{N}$: $r((y,i),(z,i))\in \Amc$ iff $r(y,z)$ is in $p(x)$;
\item for all roles $r$, variables $y,z\in V_{2}$ and $i\in I$: $r((y,i),(z,i))\in \Amc$ iff $r(y,z)$ is in $p(x)$; 
\item for all roles $r$, variables $y\in \text{FC}_{p}(x)$ and $z \in V_{1}$, and $i\in \mathbb{N}$: $r(y,(z,i))\in \Amc$ iff $r(y,z)$ is in $p(x)$;
\item for all roles $r$, variables $y\in \text{FC}_{p}(x)$ and $z \in V_{2}$, and $i\in I$: $r(y,(z,i))\in \Amc$ iff $r(y,z)$ is in $p(x)$;
\item for all roles $r$, $(y,(\beta,E))\in (V_{2} \times I)$ and $(z,i) \in V_{1}\times \mathbb{N}$: $r((y,(\beta,E)),(z,i))\in \Amc$ iff 
$i\in E$ and $\beta(i)=r(y,z)$.
\end{itemize}
This construction of the ABox $\Amc$ achieves the following:
\begin{itemize}
\item for all copies $V_{2}'$ of $V_{2}$ and any $r(y,z)\in \mathcal{S}$, there is a copy $V_{1}'$ of $V_{1}$
such that $r(y',z')\in \Amc$ for the copies $y',z'$ of $y$ and $z$ in $V_{2}'$ and $V_{1}'$, respectively;
\item for all copies $V_{2}'$ of $V_{2}$ and copies $V_{1}'$ of $V_{1}$ there is at most one atom $r(y,z)\in \Amc$
with $y\in V_{2}$ and $z\in V_{1}$;
\item Let $V_{2}^{1},\ldots,V_{2}^{n}$ be distinct copies of $V_{2}$ and $r_{1}(y_{1},z_{1}),\ldots,r_{n}(y_{n},z_{n})$
be distinct atoms in $\mathcal{S}$. Then there is a single copy $V_{1}'$ of $V_{1}$ such that $r_{1}(y_{1}',z_{1}^{1}),
\ldots,r_{n}(y_{n}',z_{n}^{n})\in \Amc$ for the copies $y_{1}',\ldots,y_{n}'$ of $y_{1},\ldots,y_{n}$ in $V_{2}^{1},\ldots,V_{2}^{n}$,
respectively, and the copies $z_{1}^{1},\ldots,z_{n}^{n}$ of $z_{1},\ldots,z_{n}$ in $V_{1}'$.  
\end{itemize} 
We first show Point~(a) above.

\medskip
Point~(a). $\Tmc,\Amc\not\models q(a)$. 

\medskip
Proof of Point~(a). By construction, all $r$ functional w.r.t.~$\Tmc$ are functional in $\Amc$. Thus, $\Tmc,\Amc\models q(x)$
iff $\emptyset,\Amc\models p'(x)$ for some disjunct $p'(x)$ of $q(x)$. It therefore suffices to show that there
is no homomorphism from any disjunct $p'(x)$ of $q(x)$ to $\Amc$ mapping $x$ to $x$. Consider the mapping 
$$
\pi: \Amc \rightarrow p(x)
$$
mapping every variable in $\text{FC}_{p}(x)$ to itself and
every $(y,i)\in (V_{1}\times \mathbb{N}) \cup (V_{2} \times I)$ to $y$.
It is easy to see that $\pi$ is a homomorphism. It also follows that there is no homomorphism from
any disjunct $p'(x)$ of $q(x)$ distinct from $p(x)$ to $\Amc$ as the composition of $\pi$ with such a homomorphism would be a
homomorphism from $p'(x)$ to $p(x)$ mapping $x$ to $x$.
It remains to prove that there is no homomorphism from $p(x)$ to $\Amc$ mapping $x$ to $x$.
Assume there is such a homomorphism $h$. Then $\pi\circ h$ is a homomorphism from $p(x)$ to $p(x)$ mapping $x$ to $x$.
We obtain a contradiction if we can show that $\pi \circ h$ is not surjective.
To this end assume that $\pi\circ h$ is surjective. As $p(x)$ is finite, it is an isomorphism.
Let $h[p(x)]=\{ h(y) \mid y \in p(x)\}$ be the image of $h$ in $\Amc$.
Then $h$ is an isomorphism from $p(x)$ onto the restriction $\Amc|_{h[p(x)]}$ of $p(x)$ to $h[p(x)]$ 
and the restriction $\pi|_{h[p(x)]}$ of $\pi$ to $h[p(x)]$ is an isomorphism onto $p(x)$. 
It follows that $h[p(x)]$ contains for every $y$ in $p(x)$ exactly one individual $a$ with $\pi(a)=y$
and 
\begin{itemize}
\item $h[p(x)]$ contains $\text{FC}_{p}(x)$;
\item as $V_{1}$ is connected, there exists $i\in \mathbb{N}$ such that $h([p(x)])\supseteq V_{1}\times \{i\}$ and
$h([p(x)])\cap V_{1}\times \{j\}=\emptyset$ for all $j\not=i$;
\item for every connected component $V$ of $V_{2}$ of there exists $i\in I$ such that 
$h([p(x)])\supseteq V\times \{i\}$ and
$h([p(x)])\cap V \times \{j\}=\emptyset$ for all $j\not=i$. 
\end{itemize}
Now recall that there there are distinct atoms $s_{1}(y_{1},z_{1})\in p(x)$ and $s_{2}(y_{2},z_{2})\in p(x)$
such that $y_{1},y_{2}$ are in the same connected component in $V_{2}$ and $z_{1},z_{2}\in V_{1}$.
Thus, there exists $i\in \mathbb{N}$ such that $(y_{1},i),(y_{2},i)\in h[p(x)]$ and for some $j\in I$,
$s_{1}((y_{1},i),(z_{1},j))\in \Amc$ and $s_{2}((y_{2},i),(z_{2},j))\in \Amc$. As observed above, no
such two atoms exist in $\Amc$ and we have derived a contradiction.
This finishes the proof of Point~(a).

\begin{example}
Let $\Tmc=\{{\sf func}(s_{1}),{\sf func}(s_{2})\}$ and
consider the CQ
$$
q(x) = \exists y,z(r(x,y) \wedge s_{1}(y,z) \wedge s_{2}(y,z))
$$
Then
$$
\text{FC}_{q}(x)=\{x\}, \quad V_{1}=\{z\}, \quad V_{2}=\{y\}
$$
Thus, 
$$
\mathcal{S} = \{ s_{1}(y,z),s_{2}(y,z)\}
$$
and so the individuals of $\Amc$ are
$$
\{x\} \cup (\{z\} \times \mathbb{N}) \cup (\{y\} \times I)
$$
and the essential properties of $\Amc$ are:
\begin{itemize}
\item $r(x,(y,i))\in \Amc$ for all $i\in I$;
\item for every $(y,i)\in I$ there are distinct $(z,i_{1}),(z,i_{2})$ with
$s_{1}((y,i),(z,i_{1})),s_{2}((y,i),(z,i_{2}))\in \Amc$;
\item for any two $(z,i_{1}),(z,i_{2})$ there exists $(y,i)$ such that
$s_{1}((y,i),(z,i_{1})),s_{2}((y,i),(z, i_{2}))\in \Amc$.
\end{itemize}
Observe that $\Amc\not\models q(x)$.
\end{example}
Points~(b) and~(c) are as before. It remains to show Point~(d). Let $\Imc$ be a model of $\Tmc$ and $\Amc$ with 
$\Imc\models \neg C(a)$ and consider the ultrafilter extension $\Imc^{\text{ue}}$. We define a homomorphism $h$ from
$p(x)$ to $\Imc^{\text{ue}}$ mapping $x$ to $\mathfrak{U}_{x}$. For every $y\in \text{FC}_{p}(x)$, we set $h(y)=\mathfrak{U}_{y}$. To define $h$ for the remaining variables, fix a non-principal ultrafilter $\mathfrak{N}$ over 
$\mathbb{N}$. For every variable $z\in V_{1}$ we obtain an ultrafilter $\mathfrak{N}(z)$ over $\Delta^{\Imc}$ by setting 
$U\in \mathfrak{N}(z)$ iff $\{i \mid (z,i)\in U\cap (V_{1}\times \mathbb{N})\}\in \mathfrak{N}$. Observe that for $z_{1},z_{2}\in V_{1}$, $r(z_{1},z_{2})\in p(x)$ implies 
$(\mathfrak{N}(z_{1}),\mathfrak{N}(z_{2}))\in r^{\Imc^{\text{ue}}}$. 
Observe as well that for $y\in \text{FC}_{p}(x)$ and $z\in V_{1}$,
$r(y,z)\in p(x)$ implies $(\mathfrak{U}_{y},\mathfrak{N}(z))\in r^{\Imc^{\text{ue}}}$.
We set $h(z)= \mathfrak{N}(z)$ for $z\in V_{1}$. It remains to define $h$ for variables in $V_{2}$.

Let for $y\in V_{2}$, and $X\subseteq \Delta^{\Imc}$, $\rho_{y}(X)=\{ i \in I \mid (y,i)\in X\}$.
By construction, the set
$$
\mathfrak{X}= 
\{ \rho_{y}(\exists r.Z)^{\Imc}\mid \text{ $r(y,z)\in \mathcal{S}$ and $Z\in \mathfrak{N}(z)$} \}
$$
has the finite intersection property. Thus, there exists an ultrafilter $\mathfrak{I}$ over $I$ containing $\mathfrak{X}$.
For every variable $y\in V_{2}$ we obtain an ultrafilter $\mathfrak{I}(y)$ over $\Delta^{\Imc}$ by setting $U\in \mathfrak{I}(y)$
iff $\{i \mid (y,i)\in U \cap (V_{2}\times I)\}\in \mathfrak{I}$. Observe that for $y_{1},y_{2}\in V_{1}$, $r(y_{1},y_{2})\in p(x)$ implies 
$(\mathfrak{I}(y_{1}),\mathfrak{I}(y_{2}))\in r^{\Imc^{\text{ue}}}$. Observe as well that for $y\in \text{FC}_{p}(x)$ and $y'\in V_{2}$,
$r(y,y')\in p(x)$ implies $(\mathfrak{U}_{y},\mathfrak{I}(y'))\in r^{\Imc^{\text{ue}}}$. 
Finally, observe that by construction $r(y,z)\in \mathcal{S}$ implies that $(\mathfrak{I}(y),\mathfrak{N}(z))\in r^{\Imc^{\text{ue}}}$.
We set $h(y) = \mathfrak{I}(y)$ for $y\in V_{2}$.
It follows that $h$ is a homomorphism from $p(x)$ to $\Imc^{\text{ue}}$, as required.

Now let $\Amc'$ be the image of $p(x)$ under $h$, regarded as an ABox. Then 
$\Tmc,\Amc'\models q(\mathfrak{U}_{x})$, as required.

The proof when $\mathcal{ALCI}$-IQ is replaced with $\mathcal{ALCI}^{u}$-IQ and connectedness is dropped
is a straightforward variation of the proof above.
\end{proof}

\thmUndecALCF*
\noindent
\begin{proof}
	We use a reduction from emptiness checking in $(\mathcal{ALCF},\text{AQ})$ which is known to be 
	undecidable~\cite{DBLP:journals/jair/BaaderBL16}. Let $Q=(\Tmc, \Sigma, A(x))$ be an OMQ from this language
	and let $q(x)=\exists y \, A(x) \wedge r(x, y) \wedge r(y,y)$, where $r$ is a role name that does not occur in $\Tmc$. 
	We show that $Q$ is empty iff $Q'=(\Tmc, \Sigma\cup \{r\}, q(x))$ is IQ-rewritable.
	Clearly, if $Q$ is empty, then $Q'$ is empty and, therefore, IQ-rewritable. Conversely, if $Q$ is not empty, then we show that $Q'$ is not IQ-rewritable. To this end we take a $\Sigma$-ABox $\Amc_{0}$ and $a\in {\sf ind}(\Amc)$ such that $\Amc_{0}\models Q(a)$ and $\Amc_{0}$
	is consistent with~$\Tmc$. Assume for a proof by contradiction that there is a rewriting $Q''=(\Tmc',\Sigma\cup\{r\},C(x))$ of $Q'$.
	We use a minor modification of the construction in the proof of Theorem~\ref{thm:f}:
    \begin{itemize}
		\item[(a)] Using the CQ $q(x)$ and $\Amc_{0}$ we construct an infinite $(\Sigma\cup\{r\})$-ABox $\Amc\supseteq \Amc_{0}$ such that $\Amc\not\models Q'(a)$; 
		\item[(b)] By compactness of FO there exists a forest model $\Imc$ of $\Tmc'$ and $\Amc$ with $\Imc\models \neg C(a)$; 
		\item[(c)] Then $\Imc^{\text{ue}}\models \neg C(\Umf_{a})$ and $\Imc\models \Tmc'$ by Lemma~\ref{lem:anti-preservation};
		\item[(d)] Moreover, $\Imc^{\text{ue}}$ is a model of a finite $(\Sigma\cup \{r\})$-ABox $\Amc'$ with individual name $\Umf_{a}$ 
		such that $\Amc'\models Q'(\Umf_{a})$. But this contradicts $\Imc^{\text{ue}}\models \neg C(\Umf_{a})$.
	\end{itemize}
	The construction of $\Amc$ is the same as in Point~1 of the proof of Theorem~\ref{thm:f} above using the query $q(x)$ and its cycle $r(y,y)$ except that we also attach $\Amc_{0}$ to the ABox constructed by identifying $x$ and $a$.
	Now the proof of Points~(a) to (d) is exactly as before.
\end{proof}
\section{Proofs for Section 6}

We start with definitions of MMSNP and CSP together with some
preliminaries. We consider signatures \Sbf that consist of predicate
symbols with unrestricted arity, known as \emph{schemas}.  An
\emph{\Sbf-fact} is an expression of the form $S(a_1,\ldots, a_n)$
where $S\in\Sbf$ is an $n$-ary predicate symbol, and $a_1, \ldots,
a_n$ are elements of some fixed, countably infinite set $\mn{const}$
of \emph{constants}. An \emph{$\Sbf$-instance} $I$ is a set of
$\Sbf$-facts. The \emph{domain} of $I$, denoted $\adom(I)$, is the set
of constants that occur in some fact in $I$.  The notions of cycles,
girth, acyclicity, and connectedness can be lifted from ABoxes to
$\Sbf$-instances, for details see \cite{DBLP:conf/icdt/FeierKL17}.  We
use $\mn{CON}_I$ to denote the set of $\Sbf$-instances that are the
maximal connected components of the \Sbf-instance $I$.

An \emph{MMSNP sentence $\phi$ over schema~$\Sbf$} has the form
$$
\exists X_1 \cdots \exists X_n \forall x_1 \cdots \forall x_m\,
\psi,
$$
with $X_1, \ldots X_n$ monadic second-order (SO) variables, $x_1,
\ldots, x_m$ first-order (FO) variables, and $\psi$ a conjunction of
quantifier-free formulas of the form 
\begin{equation}
\alpha_1 \wedge \cdots \wedge \alpha_n \rightarrow \beta_1 \vee \cdots
\vee \beta_m \text{ with } n,m \geq 0, 
\tag{$\dagger$}
\end{equation}
where each $\alpha_i$ is of the form $X_i(x_j)$ or $R(\bf{x})$ with
$R$ from $\Sbf$, and each $\beta_i$ is of the form $X_i(x_j)$. We
refer to a formula of the form $(\dagger)$ as a \emph{rule} in
$\phi$, to the 
conjunction $ \alpha_1 \wedge \cdots \wedge \alpha_n$ as its
\emph{body}, and to $\beta_1 \vee \cdots \vee \beta_m$ as its
\emph{head}. A rule body can be seen as an $\Sbf \cup \{X_1, \ldots,
X_n\}$-instance in the obvious way, which we shall sometimes do
implicitly.  An MMSNP sentence $\phi$ is \emph{connected} if the body
of every rule in $\phi$ is connected. The \emph{rule size} of $\phi$
is the maximum size of a rule in $\phi$.

Every MMSNP sentence~$\phi$ can be seen as a Boolean query in the
obvious way, that is, for an $\Sbf$-instance~$I$, $I \models \phi$
whenever $\phi$ evaluates to true on $I$.  We also consider
\emph{disjunctions of MMSNP sentences over schema ~\Sbf}, that is, sentences
of the form $\bigvee_i \phi_i$, where each $\phi_i$ is an MMSNP
sentence over $\Sbf$. For an $\Sbf$-instance~$I$, $I \models \bigvee_i
\phi_i$, whenever there exists some~$i$ such that $\phi_i$ evaluates
to true on $I$.
An MMSNP sentence $\phi_1$ over $\Sbf$ is
\emph{contained} in an MMSNP sentence $\phi_2$ over $\Sbf$,
written $\phi_1 \subseteq \phi_2$, if for every $\Sbf$-instance $I$,
$I \models \phi_1$ implies $I \models \phi_2$. We say that $\phi_1$
and $\phi_2$ are \emph{equivalent} if $\phi_1 \subseteq \phi_2$ and
$\phi_2 \subseteq \phi_1$. Containment and equivalence are defined
in the same way for disjunctions of MMSNP sentences.

A \emph{constraint satisfaction problem (CSP)} is defined by an
$\Sbf$-instance~$T$ that is called the \textit{template}. The problem
associated with $T$ is to decide whether an input $\Sbf$-instance~$I$
admits a homomorphism to $T$, denoted $I \rightarrow T$. An MMSNP
sentence~$\phi$ over schema~$\Sbf$ is said to be \emph{CSP-definable}
if there exists an $\Sbf$-template~$T$ such that for every
$\Sbf$-instance~$I$, $I \models \phi$ iff $I \rightarrow T$.
%
%
A \emph{generalized CSP} over schema~$\Sbf$ is defined by a finite set
$\Tsf=\{T_1, \ldots T_n\}$ of \Sbf-templates. For an $\Sbf$-instance
$I$, we write $I \rightarrow \Tsf$ if $I \rightarrow T_i$ for some
$i$. An MMSNP sentence $\phi$ over schema $\Sbf$ is 
\emph{definable by a generalized CSP $\Tsf$} if for every
$\Sbf$-instance $I$, $I \models \phi$ iff $I \rightarrow \Tsf$.

For two $\Sbf$-instances~$I_1$ and $I_2$ with disjoint domains, we use
$I_1 \uplus I_2$ to denote the disjoint union of $I_1$ and
$I_2$.\footnote{We
do not assume here that $I_1$ and $I_2$ contain the same nullary 
predicate symbols; $I_1 \uplus I_2$ contains the union of the
nullary symbols in $I_1$ and $I_2$.} An
MMSNP sentence $\phi$ is \emph{preserved under disjoint union} if for
all $\Sbf$-instances $I_1$ and $I_2$ (with disjoint domains), $I_1
\models \phi$ and $I_2 \models \phi$ implies $I_1 \uplus I_2 \models
\phi$.

\begin{lemma}
\label{lem:MMSNPRewrCh} 
An MMSNP sentence $\phi$ is CSP-definable iff it is definable by a
generalized CSP and preserved under disjoint union.
\end{lemma}

\noindent
\begin{proof}
  ``$\Rightarrow$''. Clear. 

  \smallskip

  ``$\Leftarrow$''. Assume that $\phi$ is definable by a
  generalized CSP $\Tsf=\{T_1, \ldots, T_n\}$ and preserved under
  disjoint union. For every $i$, $T_i \rightarrow \Tsf$ and thus $T_i
  \models \phi$. Assume w.l.o.g.\ that the domains of all templates in
  \Tsf are mutually disjoint. As $\phi$ is preserved under disjoint
  union, $\biguplus_{1 \leq i \leq n} T_i \models \phi$. Thus,
  $\biguplus_{1 \leq i \leq n} T_i \rightarrow \Tsf$ and consequently
  $\biguplus_{1 \leq i \leq n} T_i \to T_j$ for some $j$. This
  implies that $T_i \rightarrow T_j$ for every $i$ and thus the
  generalized CSP $\Tsf$ is equivalent to the CSP $T_j$, which
  finishes the proof.
\end{proof}

We next determine the complexity of deciding preservation under
disjoint union for MMSNP sentences. For our final aim, we only need
the upper bound, but we also observe a lower bound for the sake of
completeness. 
%
\begin{theorem}
\label{thm:cluduComplex}
Deciding whether an MMSNP sentence is preserved under disjoint union is 2\NExpTime-complete. 
\end{theorem}
\noindent
\begin{proof}
  For the upper bound, we reduce preservation under disjoint union to
  a series of (exponentially many) containment checks between MMSNP
  sentences (of polynomial size) and then invoke the result from
  \cite{DBLP:conf/kr/BourhisL16} that MMSNP containment can be decided
  in 2\NExpTime. 

  Let $\vp$ be an MMSNP sentence over schema \Sbf and $\mathcal{N}$
  the set of nullary predicate symbols in $\phi$. We assume w.l.o.g.\
  that the number of first-order variables in $\vp$ is bounded from
  below by the largest arity of a predicate in \Sbf. For all
  $\Nmc_1,\Nmc_2 \subseteq \Nmc$, we construct an MMSNP sentence
  $\phi_{\Nmc_1,\Nmc_2}$ as follows:

\begin{itemize}

\item $\phi_{\Nmc_1,\Nmc_2}$ has the same quantifiers as $\phi$ except
  for two new existentially quantified second-order variables $C_1$
  and $C_2$;

\item $\phi_{\Nmc_1,\Nmc_2}$ has the following rules: 

\begin{itemize}
\item $\mn{true} \to C_1(x) \vee C_2(x)$;
\item $C_1(x) \wedge C_2(x) \to \mn{false}$;
\item $R(y_1, \ldots, y_n) \wedge C_i(y_j) \to C_i(y_k)$ whenever $R \in
  \Sbf$ is $n$-ary, $i\in \{0, 1\}$, and $j,k \in \{1,\dots,n\}$ and
 where $y_1,\dots,y_n$ are the first $n$ FO variables in $\vp$;
\item $C_i(x_1)\wedge \cdots \wedge C_i(x_n) \wedge \mn{body} \to
  \mn{head}$ whenever $\mn{body} \to \mn{head}$ is a rule in $\vp$
  with no predicate symbol from $\Nmc \setminus \Nmc_i$ occuring in
  $\mn{body}$ and $i \in \{0,1\}$ and where $x_1, \ldots, x_n$ are the
  FO variables in $\mn{body}$.
\end{itemize}
\end{itemize}
Intuitively, an \Sbf-instance $I$ satisfies $\phi_{\Nmc_1,\Nmc_2}$ iff
there is a coloring of $I$ with the two colors $C_1$ and $C_2$ such
that elements from the same maximal connected components receive the
same color and each of the resulting two monochromatic subinstances of
$I$ satisfies $\phi$. Note that $I$ is the disjoint union of $I_1$ and
$I_2$. The sets $\Nmc_1,\Nmc_2$ help to disentangle the nullary
predicate symbols: $\mathcal{N}_1$ contains the predicates true in the
monochromatic subinstance colored $C_1$ and likewise for $\Nmc_2$ and
$C_2$.
\begin{claim}
$\phi$ is preserved under disjoint union iff $\phi \equiv \bigvee_{\Nmc_1,\Nmc_2 \subseteq \mathcal{N}} \phi_{\Nmc_1,\Nmc_2}$. 
\end{claim}

``$\Rightarrow$'. Assume that $\phi$ is preserved under
disjoint union. We have to show the following inclusions:
\begin{itemize}

\item $\phi \subseteq \bigvee_{\Nmc_1,\Nmc_2 \subseteq \mathcal{N}}
  \phi_{\Nmc_1,\Nmc_2}$. Let $I\models\phi$. Further let $I_1$ be the
  extension of $I$ in which every element is colored $C_1$ and let
  $\Nmc_1$ be the set of nullary predicates true in $I$ and
  $\Nmc_2=\emptyset$. Clearly, $I_1 \models \vp_{\Nmc_1,\Nmc_2}$,
  witnessing $I \models \bigvee_{\Nmc_1,\Nmc_2 \subseteq \mathcal{N}} \phi_{\Nmc_1,\Nmc_2}$.

\item $\bigvee_{\Nmc_1,\Nmc_2 \subseteq \mathcal{N}}
  \phi_{\Nmc_1,\Nmc_2} \subseteq \phi$. Let $I \models
  \bigvee_{\Nmc_1,\Nmc_2 \subseteq \mathcal{N}}
  \phi_{\Nmc_1,\Nmc_2}$. Then $I \models \phi_{\Nmc_1,\Nmc_2}$ for
  some $\Nmc_1,\Nmc_2 \subseteq \Nmc$. By construction of
  $\phi_{\Nmc_1,\Nmc_2}$, there is thus a partition of $\mn{dom}(I)$
  into two sets $S_1,S_2$ such that $I_1 \models \phi$ and $I_2
  \models \phi$ where $I_i$ is the restriction of $I$ to domain $S_i$
  and makes exactly the nullary predicates in $\Nmc_i$ true, $i \in
  \{1,2\}$. As $I$ is the disjoint union of $I_1$ and $I_2$, $I
  \models \phi$.

\end{itemize}

``$\Leftarrow$''. Assume that $\phi \equiv \bigvee_{\Nmc_1,\Nmc_2
  \subseteq \mathcal{N}} \phi_{\Nmc_1,\Nmc_2}$. Let $I_1$ and $I_2$ be
$\Sbf$-instances with disjoint domain such that $I_1 \models \phi$ and
$I_2 \models \phi$. Then $I_1 \uplus I_2 \models
\phi_{\Nmc_1,\Nmc_2}$, where $\Nmc_i$ is the set of nullary predicates
true in $I_i$, $i \in \{1,2\}$. Thus, $I_1 \uplus I_2 \models
\bigvee_{\Nmc_1,\Nmc_2 \subseteq \mathcal{N}} \phi_{\Nmc_1,\Nmc_2}$
and from the original assumption $I_1 \uplus I_2 \models \phi$.
This finishes the proof of the claim. 

\medskip

It remains to note that the inclusion $\phi \subseteq
\bigvee_{\Nmc_1,\Nmc_2 \subseteq \mathcal{N}} \phi_{\Nmc_1,\Nmc_2}$
holds even when $\varphi$ is not preserved under disjoint union (as
shown by the proof above) and thus deciding whether $\phi$ is
preserved under disjoint union amounts to checking that $\phi
\supseteq \phi_{\Nmc_1,\Nmc_2}$ for all $\Nmc_1,\Nmc_2 \subseteq
\mathcal{N}$. This gives the desired upper bound.

\medskip

For the lower bound, we consider the (polynomial time) reduction from
\cite{DBLP:conf/kr/BourhisL16} from a 2\NExpTime-hard torus tiling
problem to OMQ containment: there, two OMQs $Q_1$ and $Q_2$ are
constructed, $Q_1$ from $(\ALCI, \text{BAQ})$ and $Q_2$ from $(\ALCI,
\text{BCQ})$, such that a double exponentially large torus can be
tiled iff $Q_1 \subseteq Q_2$. We sketch a polynomial time reduction
from the containment problem for two such OMQs to the preservation
under disjoint union of an MMSNP sentence $\phi$. For the sake of
proving the correctness of our reduction, we note that $Q_1$ and $Q_2$
are such that $Q_2 \not \subseteq Q_1$. 

It was shown in \cite{DBLP:journals/tods/BienvenuCLW14} that every OMQ
from $(\ALCI, \text{BAQ})$ such as $Q_1$ is equivalent to the
complement of a CSP $T$ in the sense that for every $\Sigma$-ABox
$\Amc$, $\Amc \models Q$ iff $\Amc \not \rightarrow \Bmc$ where
$\Sigma$ is the signature of the OMQs $Q_1$ and $Q_2$. Note that this
is a variation of Lemma~\ref{lem:ToBeDone} for the case of BAQs. It
has also been shown in \cite{DBLP:journals/tods/BienvenuCLW14} that
for every OMQs from $(\ALCI, \text{BCQ})$, one can construct in
polynomial time an MMSNP sentence whose complement is equivalent to
the OMQ and thus for $Q_1$ and $Q_2$ we find two such sentences
$\phi_1$ and $\phi_2$. The size of $\phi_1$ and $\phi_2$ is polynomial
in that of $Q_1$ and $Q_2$. Summing up, the 2-exp torus can be tiled
iff $\phi_2 \subseteq \phi_1$, $\phi_1$ is equivalent to a CSP, and
$\phi_1 \not\subseteq \phi_2$. We next construct an MMSNP sentence
$\phi$ such that for all $\Sbf$-instances $I$,
\begin{equation*}
I \not \models \phi \mbox{ iff } I  \not \models \phi_1 \mbox{ and }
I \not \models \phi_2.
\tag{$\ddagger$}
\end{equation*}
%
%
Towards constructing $\phi$, we start by standardizing apart all FO
and SO variables from $\phi_1$ and $\phi_2$. For each $i \in \{1,
2\}$, let $\psi^-_i$ be the quantifier-free part of $\phi_i$ with all
rules of the form $\mn{body} \rightarrow \mn{false}$ removed. Then
$\phi$ is the MMSNP sentence which has as SO/FO variables the union of
SO/FO variables from $\phi_1$ and $\phi_2$ and the following rules:
\begin{itemize}

\item all rules from $\psi_1^-$ and from $\psi_2^-$, 

\item all rules of the form $\mn{body}_1 \wedge \mn{body}_2
  \rightarrow \mn{false}$, where $\mn{body}_1 \rightarrow
  \mn{false}$ is a rule in $\phi_1$ and $\mn{body}_2
  \rightarrow\mn{false}$ a rule in $\phi_2$. 

\end{itemize}
It can be verified that $\phi$ satisfies $(\ddagger)$.

\begin{claim}
$\phi_2 \subseteq \phi_1$ iff $\phi$ is preserved under disjoint union. 
\end{claim}


``$\Rightarrow$''. Assume that $\phi_2 \subseteq \phi_1$. Then, $\phi$
is equivalent to $ \phi_1$ and thus to a CSP, consequently it is
preserved under disjoint union.

\smallskip

``$\Leftarrow$''. Assume that $\phi_2 \not \subseteq \phi_1$. Thus
there exist instances $I_1$ and $I_2$ such that $I_1 \models \phi_1$,
$I_1 \not \models \phi_2$, $I_2 \models \phi_2$, $I_2 \not \models
\phi_1$. From ($\ddagger$), we obtain $I_1 \models \phi$ and $I_2
\models \phi$. We next observe that $I_1 \not\models \vp_1$ and $I_2
\not\models \vp_2$ implies $I_1 \uplus I_2 \not \models \phi_1$ and
$I_1 \uplus I_2 \not \models \phi_2$. Thus $I_1 \uplus I_2 \not
\models \phi$ by ($\ddagger$). Consequently, $I_1$ and $I_2$ witness
that $\vp$ is not preserved under disjoint union. 
\end{proof}
We next characterize the equivalence of MMSNP sentences to a
generalized CSP and analyze the complexity of deciding this property.

For an MMSNP sentence $\phi$, let $\phi_{\mn{acyc}}$ be the MMSNP
sentence with the same quantifier prefix that contains all rules
which have an acyclic body and can be obtained from a rule in $\phi$
by zero or more identifications of variables.
\begin{theorem} \label{thm:MMSNPGenCSP} An MMSNP sentence $\vp$ is
  definable by a generalized CSP iff $\phi \equiv
  \phi_{\mn{acyc}}$. 
\end{theorem}

\noindent
\begin{proof}
  ``only if''. Assume that $\phi$ is definable by a generalized CSP
  $\Tsf=\{T_1, \ldots, T_n\}$. Using the construction of $\phi_{\mn{acyc}}$,
  it can be verified that $\phi \subseteq \phi_{\mn{acyc}}$. It thus remains
  to be shown that
  $\phi_{\mn{acyc}} \subseteq \phi$. If the body of each rule in
  $\phi$ is acyclic, then this is clearly the case. Otherwise, let $g$
  be the maximum girth
  of a cyclic rule body from $\phi$. Take an
  $\Sbf$-instance $I$ such that $I \not\models \phi$. We have to
  show that $I \not\models \phi_{\mn{acyc}}$. Since $\phi$ is
  equivalent
  to $\Tsf$, $I \not\to
  T_i$ for $1\leq i \leq n$. From
  Lemma~\ref{lem:sparse}\footnote{Note that the lemma in its original
    formulation in \cite{DBLP:journals/siamcomp/FederV98} applies to
    instances over schemas of any arity, not just to
    ABoxes.}, we obtain an
  $\Sbf$-instance $I^g$ of girth exceeding $g$ such that $I^g \to I$ and $I^g
  \not\to T_i$ for $1 \leq i \leq n$. Thus $I^g \not\models \phi$. As the
  girth of $I^g$ is higher than the girth of every cyclic rule body in
  $\phi$, it follows that $I^g \not\models \phi_{\mn{acyc}}$. Since $I^g 
  \rightarrow I$, $I \not\models \phi_{\mn{acyc}}$.

  ``if". Assume that $\phi \equiv \phi_{\mn{acyc}}$. Since the rule
  bodies in $\phi_{\mn{acyc}}$ are acyclic, it is easy to convert
  $\phi_{\mn{acyc}}$ into an equivalent MMSNP sentence in which each
  rule body contains at most one
  atom that uses a predicate symbol from \Sbf; see
  \cite{DBLP:journals/siamcomp/FederV98}. It is implicit in that
  paper (see also
  \cite{DBLP:journals/tods/BienvenuCLW14}) that MMSNP sentences of
  this kind have the same expressive power as generalized CSPs.
  Thus,
  $\phi$ is equivalent to a generalized CSP. 
\end{proof}

Before showing the main complexity result of this section, we state a
slightly refined version of a theorem from
\cite{DBLP:conf/kr/BourhisL16} regarding the complexity of MMSNP
containment. It emphasizes that the complexity of containment is
double exponential only in the size of the rules, but not in their
number. This only requires a careful analysis of the
constructions in \cite{DBLP:conf/kr/BourhisL16}. 
\begin{theorem}
\label{thm:containmentRel}
Containment between MMSNP sentences is in 2\NExpTime. More precisely,
for MMSNP sentences $\phi_1$ and $\phi_2$ where $\phi_i$ has $n_i$
rules and rule size $r_i$, $i \in \{1, 2\}$, it can be decided in time
$2^{2^{p(\mn{log} n_1 + r_1 +  \mn{log}\mn{log} n_2 + \mn{log} r_2)}}$. 
%
 \end{theorem}

\thmMMSNPrewr*

\noindent
\begin{proof}
  We start with the upper bound. Lemma \ref{lem:MMSNPRewrCh} and
  Theorem \ref{thm:MMSNPGenCSP} suggest an algorithm for deciding
  CSP-definability of an MMSNP sentence $\phi$: check whether $\phi$
  is preserved under disjoint union and $\phi \equiv
  \phi_{\mn{acyc}}$. The first condition can be decided in 2\NExpTime
  according to Theorem \ref{thm:cluduComplex}. As for the second
  check, we note that the size of $\phi_{\mn{acyc}}$ might be
  exponential in the size of $\phi$ so we cannot apply the MMSNP
  containment result from \cite{DBLP:conf/kr/BourhisL16}
  straightaway. However, the rule size of $\phi_{\mn{acyc}}$ is
  polynomial in the size of $\phi$ and thus by Theorem
  \ref{thm:containmentRel} the second condition can be decided in
  2\NExpTime as well.

\smallskip

For showing that CSP-definability of MMSNP sentences is
2\NExpTime-hard, we can apply the same reduction as in the proof of
Theorem \ref{thm:cluduComplex}: the MMSNP sentences $\phi_1$,
$\phi_2$, and $\phi$, constructed in the reduction are such that
$\phi_2 \subseteq \phi_1$ iff $\phi$ is equivalent to a CSP. 
\end{proof}
\end{document}